\documentclass{llncs}

\usepackage{paralist}
\usepackage{url}
\usepackage{amsmath}
\usepackage{amsthm}
\usepackage{mathtools}
\usepackage{amsfonts}
\usepackage{bm}
\usepackage[ampersand]{easylist}
\usepackage[pdf]{pstricks}
\usepackage{mathtools}

\usepackage{graphicx} % more modern
\usepackage[compatibility=false]{caption}
\usepackage{subcaption}
\usepackage{algorithm}
\usepackage{algpseudocode}
\newcommand{\algorithmicoutput}{\textbf{Output:}}
\newcommand{\Output}{\item[\algorithmicoutput]}

\newcommand{\mathbfup}[1]{\mathbf{\bm{#1}}}

\titlerunning{Relative Comparison Kernel Learning with Auxiliary Kernels}

\begin{document} 

%\twocolumn[
\title{Relative Comparison Kernel Learning with Auxiliary Kernels}

% It is OKAY to include author information, even for blind
% submissions: the style file will automatically remove it for you
% unless you've provided the [accepted] option to the icml2013
% package.
\author{Eric Heim\inst{1} \and Hamed Valizadegan\inst{2} \and Milos Hauskrecht\inst{1}}
\institute{University of Pittsburgh, Department of Computer Science \\ \email{\{eth13, milos\}@cs.pitt.edu} \and NASA Ames Research Center \\ \email{hamed.valizadegan@nasa.gov}}
%\icmlauthor{Eric Heim}{eth13@pitt.edu}
%\icmladdress{Computer Science Department, University of Pittsburgh, Pittsburgh, PA, 15213}
%\icmlauthor{Hamed Valizadegan}{hamed@cs.pitt.edu}
%\icmladdress{Computer Science Department, University of Pittsburgh, Pittsburgh, PA, 15213}
%\icmlauthor{Milos Hauskrecht}{milos@cs.pitt.edu}
%\icmladdress{Computer Science Department, University of Pittsburgh, Pittsburgh, PA, 15213}

% You may provide any keywords that you 
% find helpful for describing your paper; these are used to populate 
% the "keywords" metadata in the PDF but will not be shown in the document
%\icmlkeywords{multiple kernel learning, relative comparisons}

%\newtheorem{totalNumTriplets}{Theorem}
%\newtheorem{numTriplets}[totalNumTriplets]{Theorem}
%\newtheorem{noRank}[totalNumTriplets]{Theorem}
%\setcounter{totalNumTriplets}{0}
%\newtheorem{convexRCKL-AK}{Proposition}
%\newtheorem{convexESTE}[convexRCKL-AK]{Proposition}
%\newtheorem{convexEGNMDS}[convexRCKL-AK]{Proposition}
%\newtheorem{convexSTE-AK}{Proposition}
%\newtheorem{convexGNMDS-AK}[convexSTE-AK]{Proposition}
%\newtheorem{metricLearning}[convexRCKL-AK]{Proposition}
%\setcounter{convexRCKL-AK}{0}
%\newtheorem{log-convexity}{Definition}
%\setcounter{log-convexity}{0}
%\newtheorem{sumConvex}{Lemma}
%\newtheorem{sumlog}[sumConvex]{Lemma}
%\newtheorem{pointwiseMax}[sumConvex]{Lemma}
%\newtheorem{affineMap}[sumConvex]{Lemma}
%\setcounter{sumConvex}{0}
%\setcounter{convexSTE-AK}{0}
%\setcounter{metricLearning}{0}

\maketitle

\begin{abstract}
In this work we consider the problem of learning a positive semidefinite kernel matrix from relative comparisons of the form: ``object $A$ is more similar to object $B$ than it is to $C$'', where comparisons are given by humans.  Existing solutions to this problem assume many comparisons are provided to learn a high quality kernel. However, this can be considered unrealistic for many real-world tasks since a large amount of human input is often costly or difficult to obtain.  Because of this, only a limited number of these comparisons may be provided.  We propose a new kernel learning approach that supplements the few relative comparisons with ``auxiliary'' kernels built from more easily extractable features in order to learn a kernel that more completely models the notion of similarity gained from human feedback.  Our proposed formulation is a convex optimization problem that adds only minor overhead to methods that use no auxiliary information.  Empirical results show that in the presence of few training relative comparisons, our method can learn kernels that generalize to more out-of-sample comparisons than methods that do not utilize auxiliary information, as well as similar methods that learn metrics over objects.
%However, in many real-world learning scenarios, objects have other, easily extractable, information regarding their similarity to other objects, such as feature representations.  Our method utilizes these sources of auxiliary information to supplement the potentially noisy and incomplete set of relative comparisons in order to learn a kernel that can generalize to out-of-sample comparisons.  We do this by learning a conic combination of kernels built from the auxiliary information and a unique kernel that complements the others.  Our method is convex and efficient to solve, relative to similar models.  Empirical results show that in the presence of few training relative comparisons, our method learns kernels that can generalize to more out-of-sample comparisons than methods that do not utilize auxiliary information, as well as similar metric learning methods.  
\end{abstract}

\section{Introduction}

\label{sec:intro}
The effectiveness of many kernel methods for unsupervised \cite{valizadegan2006generalized,filippone2008survey}, semi-supervised \cite{zhu2002learning,zhou2004learning,wang2008label}, and supervised \cite{scholopf2002learning} learning is highly dependent how meaningful the input kernel is for modeling similarity among objects for a given task.  In practice, kernels are often built by using a standard kernel function on features extracted from data.  For example, when building a kernel over clothing items, features can be extracted for each item regarding attributes like size, style, and color.  Then, a predefined kernel function (e.g. the Gaussian kernel function) can be applied to this feature representation to build a kernel over clothing.  However, for certain tasks, objects may not be represented well by extracted features alone.  Consider a product recommendation system for suggesting replacements for out of stock clothing items.  Such a system requires a model of similarity based on how humans perceive clothing, which may not be captured entirely by features.  For this, it is likely that human input is necessary to construct a meaningful kernel. 

%Certain tasks, however, require a notion of similarity built from human perception, and cannot be built from raw data alone.  For instance, in order for a product reccomendation system to suggest a suitable replacement for a clothing item that is out of stock, it needs to have a model of similarity based on how humans perceive clothing. In addition, certain tasks not entirely reliant on human similarity have been shown to benefit from a human perspective of how objects relate, such as various computer vision tasks \cite{kumar2009attribute,branson2010visual,parikh2011relative,duan2012discovering}.  

In general, obtaining reliable information from humans can be challenging, but retrieving \emph{relative comparisons} of the form ``object $A$ is more similar to object $B$ than it is to object $C$'' has several attractive characteristics.  First, relative comparison questions are less mentally fatiguing to humans compared to other forms (e.g. quantitative comparison questions: ``On a scale from 1 to 10, how similar are objects $A$ and $B$?'') \cite{kendall1990rank}.  Second, there is no need to reconcile individual humans' personal scales of similarity. Finally, relative comparison feedback can be drawn implicitly through certain human-computer interactions, such as mouse clicks.  In this work, we consider the specific problem of learning a kernel from relative comparisons;  A problem we will refer to as \emph{relative comparison kernel learning} (RCKL).
 
Current RCKL methods \cite{agarwal2007generalized,tamuz2011adaptively,van2012stochastic} assume that all necessary human feedback to build a useful kernel is provided. This is often not the case in real-world scenarios.  A large amount of feedback is needed to build a kernel that represents a meaningful notion of how a human views the relationships among objects, and obtaining feedback from humans is often difficult or costly.  Hence, it is a realistic assumption that only a limited amount of feedback can be obtained.

In order to learn a meaningful kernel from only a limited number of relative comparisons we propose a novel RCKL method that learns a combination of \emph{auxiliary kernels} built from extracted features and a \emph{learned kernel}. The intuition behind this approach is that while human feedback is necessary to construct an appropriate kernel, some aspects of how humans view similarity among objects are likely captured in easily extractable features.  If ``auxiliary'' kernels are built from these features, then they can be used to reduce the need of many relative comparisons.  To learn the aforementioned combination, we formulate a convex optimization that adds a only small amount of computational overhead to traditional RCKL methods.  Experimentally, we show that by using auxiliary information, our method can learn a kernel that accurately models the relationship among objects from few relative comparisons, including relationships among objects not explicitly given.  More specifically, our method is shown to generalize to more held out relative comparisons than traditional RCKL methods.  In addition, we compare our method to similar, state-of-the-art methods in metric learning, and, again, we show that our method generalizes to more held out relative comparisons.
   
The remainder of the paper is organized as follows.  Section \ref{sec:pre} provides our formal definition of RCKL.  Section \ref{sec:theory} motivates our problem.  Section \ref{sec:learn} introduces a general framework for extending RCKL methods to use auxiliary information. Section \ref{sec:related} overviews related work.  Section \ref{sec:experiments} presents an evaluation of our method.  Section \ref{sec:future} concludes with future work.

%Of course, some of these auxiliary sources of information may not reflect the notion of similarity provided by the relative comparisons.  Our method selects only the most ``helpful'' kernels from the full set in order to reduce this potential noise.

 \section{Preliminaries}
 \label{sec:pre}
The RCKL problem considered in this work is defined by a set of $n$ objects, $\mathcal{X} = \{\mathit{x}_1, ..., \mathit{x}_n\} \subseteq \mathbb{X}$, where $\mathbb{X}$ is the set of all possible objects.  Similarity information among objects is given in the form of a set $\mathcal{T}$ of triplets:
 
 \begin{equation}
   \mathcal{T} = \{(a,b,c) | \mathit{x}_a \mathrm{\ is\ more\ similar\ to\ }\mathit{x}_b \mathrm{\ than\ }\mathit{x}_c\}
   \label{eq:TripletsDef}
 \end{equation}
 
 The goal is to find a positive semidefinite (PSD) kernel matrix $\mathbf{K} \in \mathbb{R}^{n \times n}$ that satisfies the following constraints:
 
 \begin{equation}
   \begin{array}{rl}
       \forall_{\left(a,b,c\right) \in \mathcal{T}} : & d_{\mathbf{K}}(\mathit{x}_a, \mathit{x}_b) < d_{\mathbf{K}}(\mathit{x}_a, \mathit{x}_c)\\ [5 pt]
      \mathrm{Where} & d_{\mathbf{K}}(\mathit{x}_a, \mathit{x}_b) = \mathbf{K}_{aa} + \mathbf{K}_{bb} - 2\mathbf{K}_{ab}
   \end{array}
   \label{eq:KConstraints}
 \end{equation}
 
Here, $\mathbf{K}_{ab}$ is the element in the $a$th row and $b$th column of $\mathbf{K}$, representing the similarity between the $a$th and $b$th objects.  The elements of $\mathbf{K}$ can be interpreted as the inner products of the objects embedded in a Reproducing Kernel Hilbert Space (RKHS), $\mathcal{H}_{\mathbf{K}}$, endowed with a mapping $\mathbfup{\Phi}_{\mathbf{K}} : \mathbb{X} \rightarrow \mathcal{H}_{\mathbf{K}}$.  With this interpretation $\mathbf{K}_{aa} + \mathbf{K}_{bb} - 2\mathbf{K}_{ab} = \|\mathbfup{\Phi}_{\mathbf{K}}(x_a) - \mathbfup{\Phi}_{\mathbf{K}}(x_b)\|_2^2$.  Thus, learning a kernel matrix $\mathbf{K}$ that satisfies the constraints in \eqref{eq:KConstraints} is equivalent to embedding the objects in a space, such that for all triplets $(a,b,c) \in \mathcal{T}$, $\mathit{x}_a$ is closer to $\mathit{x}_b$ than it is to $\mathit{x}_c$ without explicitly learning the mapping $\mathbfup{\Phi}_{\mathbf{K}}$.  We say that a triplet $\left(a,b,c\right)$ is \emph{satisfied} if the corresponding constraint in \eqref{eq:KConstraints} is satisfied.

One interpretation of \eqref{eq:KConstraints} is that triplets define a binary relation over distances among objects.  For example, if $\mathcal{X} = \{x_1,x_2,x_3\}$ and $\mathcal{T} = \{(1,2,3),$ $(2,1,3)\}$, then $\mathcal{T}$ defines a relation, $R_{\mathcal{T}}$ over $\mathcal{S}_{\mathcal{X}} = \{(d_{\mathbf{K}}(\mathit{x}_1, \mathit{x}_2), d_{\mathbf{K}}(\mathit{x}_1, \mathit{x}_3),$ $d_{\mathbf{K}}(\mathit{x}_2, \mathit{x}_3))\}$, such that $R_{\mathcal{T}} = \{(d_{\mathbf{K}}(\mathit{x}_1, \mathit{x}_2), d_{\mathbf{K}}(\mathit{x}_1, \mathit{x}_3)),(d_{\mathbf{K}}(\mathit{x}_1, \mathit{x}_2),d_{\mathbf{K}}(\mathit{x}_2, \mathit{x}_3))\}$.  With this in mind, we continue onto the next section where we discuss the RCKL problem in more depth.

\section{The Impact of Few Triplets}
\label{sec:theory}
To help motivate this work, we provide some insight into why it can be assumed, in practice, that only a limited number of triplets can be obtained from humans, and the potential impact it has on learning an accurate kernel.  First, we begin by defining some properties of sets of triplets:

\begin{definition}
\label{def:transClo}
Given a set of triplets $\mathcal{T}$, let $\mathcal{T}^{\infty}$ be the \emph{transitive closure} of $\mathcal{T}$ 
\end{definition}

\begin{definition}
\label{def:trans}
Given a set of triplets $\mathcal{T}$, let $\mathcal{T}^{trans} = \mathcal{T}^{\infty} \setminus \mathcal{T}$ 
\end{definition}

Definition \ref{def:trans} simply defines $\mathcal{T}^{trans}$ as the set of triplets that can be inferred by transitivity of triplets in $\mathcal{T}$.  For example, if $\mathcal{T} = \{(a,b,c), (c,a,b)\}$ then $\mathcal{T}^{trans} = \{(b,a,c)\}$.

\begin{definition}
A set $\mathcal{T}$ of triplets is \emph{conflicting} if $\exists a,b,c : (a,b,c) \in \mathcal{T}^{\infty} \wedge (a,c,b) \in \mathcal{T}^{\infty}$ 
\end{definition}

A set of conflicting triplets given by a source can be seen as inconsistent or contradictory in terms of how the source is comparing objects.  In practice, this can be handled by prompting the source of triplets to resolve this conflict or by using simplifying methods such as in \cite{mcfee2011learning}.  We defer to these methods in terms of how conflicts can be dealt with and consider the \emph{non-conflicting} case.  Let $\mathcal{T}^{total}$ be the set of all non-conflicting triplets that would be given by a source, if prompted with every relative comparison question over $n$ objects.  We begin by stating the following:

\begin{theorem}
\label{thm:totalNumTriplets}
For $n$ objects, $|\mathcal{T}^{total}| = \frac{1}{2}(n^3-3n^2+2n)$
\end{theorem}

Theorem \ref{thm:totalNumTriplets} is proven in Sec. \ref{sec:appTheorem1} of the appendix.  For even a small number of objects, obtaining most of $\mathcal{T}^{total}$ from humans would be too difficult or costly in many practical scenarios, especially if feedback is gained through natural use of a system, such as an online store, or if feedback requires an expert's opinion, such as in the medical domain.  Let $\mathcal{T} \subseteq \mathcal{T}^{total}$ be the set of triplets actually \emph{obtained} from a source.  We say that a triplet $t$ is \emph{unobtained} if $t \in \mathcal{T}^{total} \setminus \mathcal{T}$.  To build a model that accurately reflects the true notion of similarity given by a source of triplets, an RCKL method should learn a kernel $\mathbf{K}$ that not only satisfies the obtained triplets, but also many of the triplets in $\mathcal{T}^{total}$, including those that were unobtained.  This means that given small number of obtained triplets, an RCKL method should somehow \emph{infer} a portion of the unobtained triplets in order to build an accurate model of similarity.  In the remainder of this section we consider two possible scenarios where unobtained triplets could potentially be inferred.

%While much analysis in kernel learning focuses on comparing a learned kernel to the ``ground truth'', we choose to compare sets of given triplets and what triplets can be infered from them.  We do this because triplets provide comparisons among distances between objects, not the distances themselves, there is often not a single ``ground truth'' kernel that can satisfy all of a given set of triplets, but many.
%In other words, if there are only few triplets obtained, and not many unobtained triplets can be infered, then an RCKL method does not have enough information about how objects relate to build an accurate model.

For the following analysis, we assume that triplets are obtained one at a time.  Also, we assume that the order in which triplets are obtained is random.  This could be a reasonable assumption in applications, such as search engines, where the goal of asking relative comparison questions that are most useful in the learning process comes secondary to providing the best search results, and as such, no assumptions can be made regarding which relative comparison questions are posed to a source.  Thus, the worst-case in the following analysis is with adversarial choice of both $\mathcal{T}^{total}$ and the order in which triplets are obtained.  Let $\mathcal{T}_i$ be the set of triplets given by an adversary after $i$ triplets are given.  Under these assumptions, we state the following theorem:

\begin{theorem}
\label{thm:numTriplets}
In the worst-case, $\forall_{i=1,...,|\mathcal{T}^{total}|} : \mathcal{T}^{trans}_i \setminus \mathcal{T}_i = \emptyset$
\end{theorem}

Theorem \ref{thm:numTriplets} is proven in Sec. \ref{sec:appTheorem2} of the appendix.  This states that in the worst case, no unobtained triplet can inferred by transitive relationship among obtained triplets.  As a result, it may fall on the RCKL methods themselves to infer triplets.  Many RCKL methods attempt to do this by assuming the learned kernel $\mathbf{K}$ has low rank.  By limiting the rank of $\mathbf{K}$ to $r < n$, an RCKL method may effectively infer unobtained triplets by eliminating those that cannot be satisfied by a rank $r$ kernel.  For instance, assume an RCKL method attempts to learn a rank $r$ kernel from $\mathcal{T}$, and assume the triplets $(a,b,c)$ and $(a,c,b)$ are not in $\mathcal{T}$.  If the set $\mathcal{T} \cup (a,c,b)$ cannot be satisfied by a rank $r$ kernel, but $\mathcal{T} \cup (a,b,c)$ can, then the RCKL method can only learn a kernel in which $(a,b,c)$ is satisfied.  Let $\mathcal{T}^{rank-r}$ be the set of all unobtained, not otherwise inferred, triplets that are inferred when an RCKL method enforces $\mathrm{rank}(\mathbf{K}) \leq r$.  For adversarial choice of $\mathcal{T}^{total}$ we can state the following theorem: 

\begin{theorem}
\label{thm:noRank}
In the worst case,  $\forall_{t \in \mathcal{T}^{rank-r}} : t \notin \mathcal{T}^{total}$
\end{theorem}

Theorem \ref{thm:noRank} is proven in Sec. \ref{sec:appTheorem3} of the appendix.  This theorem shows that it could be the case that any triplet inferred by limiting the rank of $\mathbf{K}$ is not a triplet a source would give.  If a large portion of $\mathcal{T}^{total}$ cannot be obtained or correctly inferred, then much of the information needed for an RCKL method to learn a kernel that reflects how a source views the relationship among objects is simply not available.  The goal of this work is to use auxiliary information describing the objects to supplement obtained triplets in order to learn a kernel that can satisfy more unobtained triplets than traditional methods.  In the following section we propose a novel RCKL method that extends traditional RCKL methods to use auxiliary information.

%Most RCKL methods combat these issues by assuming the learned kernel, $\mathbf{K}$, has rank $r < n-1$.  This is based on the principle that a simple $K$ is less prone to overfitting to $\mathcal{T}$.  In a sense, by limiting the rank of $\mathbf{K}l$, one can infer triplets by eliminating those that cannot be satisfied by an $r$ rank kernel.  However, the lowest rank kernel that can satisfy all the triplets in $\mathcal{T}^{total}$ could have rank as high as $n-1$ \cite{mcfee2011learning}.  If a kernel that satisfies $\mathcal{T}^{total}$ is required to have rank $n-1$, then any RCKL method that learns an $r$ rank kernel cannot guarantee a solution that satisfies any unobtained triplets.
 \section{Learning a Kernel with Auxiliary Information}
 \label{sec:learn}
In this section we introduce a generalized framework for traditional RCKL methods.  Then, we expand upon this to create two new frameworks:  One that combines auxiliary kernels to satisfy triplets, and another that is a hybrid of the previous two.

 \subsection{Traditional RCKL}
 Many RCKL methods can be generalized by the following optimization problem:

 \begin{equation}
   \begin{array}{rl}
     \displaystyle \min_{\mathbf{K}} & \displaystyle E(\mathbf{K}, \mathcal{T}) + \lambda\mathrm{trace}(\mathbf{K})\\ [5 pt]
     \mathrm{s.t.} & \mathbf{K} \succeq 0,\
   \end{array}
   \label{eq:RCKL}
 \end{equation}
 
 The first term, $E(\mathbf{K},\mathcal{T})$, is a function of the error the objective incurs for $\mathbf{K}$ not satisfying triplets in $\mathcal{T}$.  The second term regularizes $\mathbf{K}$ by its trace.  Here, the trace is used as a convex approximation of the non-convex rank function.  The rank of $\mathbf{K}$ directly reflects the dimensionality of the embedding of the objects in $\mathcal{H}_{\mathbf{K}}$.  A low setting of the hyperparameter $\lambda$ favors a more accurate embedding, while a high value prefers a lower rank kernel.  The PSD constraint ensures that $\mathbf{K}$ is a valid kernel matrix, and makes \eqref{eq:RCKL} a semidefinite program (SDP) over $n^2$ variables.  For the remainder of this paper we will refer to \eqref{eq:RCKL} as Traditional Relative Comparison Kernel Learning (RCKL-T). 
 %% For many of the current methods that can be used for RCKL, $E$ is a convex function over the optimization variable $\mathbf{K}$.  If we assume this, then we can state the following: \\

 %% \begin{convexRCKL}
 %%   If $\mathbf{E}$ is a convex function in $\mathbf{K}$, then \eqref{eq:RCKL} is a convex optimization problem.
 %% \end{convexRCKL}
 %% \begin{proof}
 %%   The trace of a matrix is a convex function and the sum of two convex functions (tr and $E$) is a convex function, thus the objective is convex.  The positive semidefinite constraint on $\mathbf{K}$ is also convex in $\mathbf{K}$ \cite{boyd2004convex}.  Both the objective function and the constraint in \eqref{eq:RCKL} are convex, so the optimization problem is convex.
 %% \end{proof}

 %% Those familiar with convex optimization will recognize this as a semidefinite program (SDP).

 %Here I took out the naming of this basic method as RCKL because it conflicts with the naming of the problem.  I previously wanted to separate this problem from NMDS as it could cause unneccessary confusion, so I called the task of learning a kernel from relative comparisons RCKL, so it should be assumed that the methods that previously performed this task are called RCKL methods.

 \subsection{RCKL via Conic Combination}
In general, if there are few triplets in $\mathcal{T}$ relative to $n$, there are many different RCKL solutions.  Without using information regarding how the objects relate other than $\mathcal{T}$, RCKL-T methods may not be able to prefer solutions that generalize well to the many unobtained triplets.  However, objects can often be described by features drawn from data.  From these features, $A$ auxiliary kernels $\mathbf{K}_1,...,\mathbf{K}_A \in \mathbb{R}^{n \times n}$ can be constructed using standard kernel functions to model the relationship among objects.  If one or more auxiliary kernels satisfy many triplets in $\mathcal{T}$, they may represent factors that influence how some of the unobtained triplets would have been answered.  For instance, if a user considers characteristics such as color and size to be important when comparing clothing items, then kernels built from color and size may model a trend in how the user answers triplets over clothing items.  If these kernels do represent a trend, then they could not only satisfy a portion of triplets in $\mathcal{T}$, but also a portion of the unobtained triplets.  We wish to identify which of the given auxiliary kernels model trends in given triplets and combine them to satisfy triplets in $\mathcal{T}$.  An approach popularized by multiple kernel learning methods is to combine kernels by a weighted sum:

 \begin{equation}
   \displaystyle \mathbf{K}' = \sum_{a=1}^A \mu_a\mathbf{K}_a\ \ \ \ \mathbfup{\mu} \in \mathbb{R}^{A}_{+}
 \label{eq:mklcomb}
 \end{equation}

$\mathbf{K}'$ is a conic combination of PSD kernels, so itself is a PSD kernel \cite{scholopf2002learning}.  $\mathbf{K}'$ induces the mapping $\mathbfup{\Phi}_{\mathbf{K}'} : \mathbb{X} \rightarrow \mathbb{R}^{D}$ \cite{gonen2011multiple}:
 
 \begin{equation}
   \displaystyle \mathbfup{\Phi}_{\mathbf{K}'}(x_i) = \lbrack \sqrt{\mu_1}\mathbfup{\Phi}_{1}(x_i),..., \sqrt{\mu_A}\mathbfup{\Phi}_{A}(x_i) \rbrack
 \label{eq:mklmap}
 \end{equation}
 
 Here $\mathbfup{\Phi}_{j} : \mathbb{X} \rightarrow \mathbb{R}^{d_j}$ is a mapping from an object into the RKHS defined by $\mathbf{K}_j$, and $D = \sum_{a=1}^{A} d_a$.  In short, \eqref{eq:mklcomb} induces a mapping of the objects into a feature space defined as the weighted concatenation of the individual kernels' feature spaces.  Consider, then, the following optimization: 
 
 \begin{equation}
   \begin{array}{rl}
     \displaystyle \min_{\mathbfup{\mu}} & \displaystyle E(\mathbf{K}', \mathcal{T}) + \lambda\|\mathbfup{\mu}\|_1\\ [5 pt]
     \mathrm{s.t.} & \mathbfup{\mu} \geq 0
   \end{array}
   \label{eq:RCKL-MKL}
 \end{equation}
 
 By learning the weight vector $\mathbfup{\mu}$, \eqref{eq:RCKL-MKL} scales the individual concatenated feature spaces to emphasize those that reflect $\mathcal{T}$ well, and reduce the influence of those that do not. Because of its relationship to multiple kernel learning, we call this formulation RCKL-MKL.

Since the auxiliary kernels are fixed, regularizing them by their traces has no effect on their rank nor the rank of $\mathbf{K}'$.  Instead, we choose to regularize $\mathbfup{\mu}$ by its $\ell_1$-norm, a technique first made popular for its use in regression \cite{tibshirani1996regression}.  For a proper setting of $\lambda$, this has the effect of eliminating the contribution of kernels that do not help in reducing the error by forcing their corresponding weights to be exactly zero.  Note that RCKL-MKL does not learn the elements of a kernel directly, and as a result is a linear program over $A$ variables.
 
 %% Like RCKL, RCKL-MKL is also convex under certain conditions:  
 %% \begin{convexRCKL-MKL}
 %%   If $E$ is a convex in $\mathbfup{\mu}$ then \eqref{eq:RCKL-MKL} is a convex optimization problem.
 %% \end{convexRCKL-MKL}
 %% \begin{proof}
 %%   Both the inequality constraint and the $\ell_1$-norm penalty are convex \cite{boyd2004convex}. (REVISE ME!)
 %% \end{proof}

By limiting the optimization to only a conic combination of the predefined auxiliary kernels, RCKL-MKL does not necessarily produce a kernel that satisfies any triplets in $\mathcal{T}$.  To capture the potential generalization power of using auxiliary information while retaining the ability to satisfy triplets in $\mathcal{T}$, we propose to learn a combination of the auxiliary kernels and $\mathbf{K}_0$, a kernel similar to the one in RCKL-T whose elements are learned directly.  By doing this, we force RCKL-T to prefer solutions more similar to the auxiliary kernels, which could satisfy unobtained triplets.  We call this hybrid approach Relative Comparison Kernel Learning with Auxiliary Kernels (RCKL-AK).

 \subsection{RCKL-AK}
 RCKL-AK learns the following kernel combination:
 
 \begin{equation}
   \displaystyle \mathbf{K}'' = \mathbf{K}_0 + \sum_{a=1}^A \mu_a\mathbf{K}_a\ \ \ \ \mathbfup{\mu} \in \mathbb{R}^{A}_{+},\ \mathbf{K}_0 \succeq 0
 \label{eq:akcomb}
 \end{equation}
 
 \eqref{eq:akcomb} is a conic combination of kernel matrices that induces the  mapping:
 
 \begin{equation}
   \displaystyle \mathbfup{\Phi}_{\mathbf{K}''}(x_i) = \lbrack \mathbfup{\Phi}_{0}(x_i), \sqrt{\mu_1}\mathbfup{\Phi}_{1}(x_i),..., \sqrt{\mu_A}\mathbfup{\Phi}_{A}(x_i) \rbrack
 \label{eq:rcklmap}
 \end{equation}
 
 The intuition behind this combination is that auxiliary kernels that satisfy many triplets are emphasized by weighing them more, and $\mathbf{K}_0$, which is learned directly, can satisfy the triplets that cannot be satisfied by the conic combination of the auxiliary kernels.  Consider, again, the example of a person comparing clothing items on an online store.  She may compare clothes by characteristics such as color, size, and material, which are features that can be extracted and used to build the auxiliary kernels.  However, other factors may influence how she compares clothes, such as designer or pattern, which may be omitted from the auxiliary kernels.  In addition, she may have a personal sense of style that is impossible to be gained from features alone.  $\mathbf{K}_0$, and thus features induced by the mapping $\mathbfup{\Phi}_{0}$, is learned to model factors she uses to compare clothes that are omitted from the auxiliary kernels or cannot be modeled by extracted features.  Using \eqref{eq:akcomb}, we propose the following optimization:
 
 \begin{equation}
   \begin{array}{rl}
     \displaystyle \min_{\mathbf{K}_0, \mathbfup{\mu}} & \displaystyle E(\mathbf{K}'', \mathcal{T}) + \lambda_1\mathrm{trace}(\mathbf{K}_0) + \lambda_2\|\mathbfup{\mu}\|_1\\ [5 pt]
     \mathrm{s.t.} & \mathbf{K}_0 \succeq 0,\ \mathbfup{\mu} \geq 0
   \end{array}
   \label{eq:RCKL-AK}
 \end{equation}
 
 This objective has two regularization terms: trace regulation on $\mathbf{K}_0$, and $\ell_1$-norm regularization on $\mathbfup{\mu}$.  Increasing $\lambda_1$ limits the expressiveness of $\mathbf{K}_0$ by reducing its rank, while increasing $\lambda_2$ reduces the influence of the auxiliary kernels by forcing the elements of $\mathbfup{\mu}$ towards zero.  Thus, $\lambda_1$ and $\lambda_2$ represent a trade-off between finding a kernel that is more influenced by $\mathbf{K}_0$ and one more influenced by the auxiliary kernels.  Like RCKL-T, RCKL-AK is an SDP, but with $n^2 + A$ optimization variables.  For practical $A$, RCKL-AK can be solved with minimal additional computational overhead to RCKL-T.

 %Note that in this work we chose $\ell_1$-norm regularization to eliminate auxiliary kernels do not reflect the triplets from the model, but other regularizers can be used in place of it for a different effect on $\mathbfup{\mu}$.

 One desirable property of \eqref{eq:RCKL-AK} is that under certain conditions, it is a convex optimization problem:
 \begin{proposition}
   If $E$ is a convex function in both $\mathbf{K}_0$ and $\mathbfup{\mu}$, then \eqref{eq:RCKL-AK} is a convex optimization problem.
 \label{prop:convexRCKL-AK}
 \end{proposition}
Proposition \ref{prop:convexRCKL-AK} is proven in Sec. \ref{sec:appProposition1} of the appendix.  While Prop. \ref{prop:convexRCKL-AK} may seem simple, it allows us to leverage traditional RCKL methods that contain error functions that are convex in $\mathbf{K}_0$ and $\mathbfup{\mu}$ in order to solve \eqref{eq:RCKL-AK} using convex optimization techniques.  Two such error functions are discussed in the following subsections.

 \subsubsection{STE-AK}
 \label{sec:STE-AK}
 Stochastic Triplet Embedding (STE) \cite{van2012stochastic} proposes the following probability that a triplet is satisfied:
 
 \begin{equation*}
   p_{abc}^{\mathbf{K}} = \frac{\mathrm{exp}(-d_{\mathbf{K}}(\mathit{x}_a, \mathit{x}_b))}{\mathrm{exp}(-d_{\mathbf{K}}(\mathit{x}_a, \mathit{x}_b)) + \mathrm{exp}(-d_{\mathbf{K}}(\mathit{x}_a, \mathit{x}_c))}
   \label{eq:STEProb}
 \end{equation*}
 
 If this probability is high, then $\mathit{x}_a$ is closer to $\mathit{x}_b$ than it is to $\mathit{x}_c$.  As such, we minimize the negative sum of the log-probabilities over all triplets.
 
 \begin{equation}
   \label{eq:STE-AK}
   \displaystyle E_{\mathrm{STE}}\left(\mathbf{K}'',\mathcal{T}\right) = -\hspace{-10 pt}\sum_{(a,b,c) \in \mathcal{T}}\hspace{-10 pt}\log(p_{abc}^{\mathbf{K}''})
 \end{equation}
 %% \begin{equation}
 %%     \displaystyle \min_{\mathbf{K}_0, \mathbfup{\mu}} & \displaystyle -\hspace{-10 pt}\sum_{(a,b,c) \in \mathcal{T}} \hspace{-5 pt} \log(p_{abc}^{\mathbf{K}''}) + \lambda_1\mathrm{trace}(\mathbf{K}_0) + \lambda_2\|\mathbfup{\mu}\|_1\\ [20 pt]
 %%     \mathrm{s.t.} & \mathbf{K}_0 \succeq 0, \mathbfup{\mu} \geq 0
 %%   \end{array}
 %%   \label{eq:STE}
 %% \end{equation}
 %
With this error function we call our method STE-AK and can state the following proposition:

 \begin{proposition}
   \eqref{eq:STE-AK} is convex in both $\mathbf{K}_0$ and $\mathbfup{\mu}$
   \label{lemma:convexESTE}
 \end{proposition}

Proposition \ref{lemma:convexESTE} is proven in Sec. \ref{sec:prop2} of the appendix.  By Props. \ref{prop:convexRCKL-AK} and \ref{lemma:convexESTE}, STE-AK is a convex optimization problem.

 %\eqref{eq:STE} is a standard semidefinite program.  We propose to solve it via the projected gradient descent algorithm outlined in Algorithm \ref{alg:STE-AK} in the supplimental materials.  
 \subsubsection{GNMDS-AK}
 \label{sec:GNMDS-AK}
 Another potential error function is one motivated by Generalized Non-Metric Multidimensional Scaling (GNMDS) \cite{agarwal2007generalized} which uses hinge loss:
 
 \begin{equation}
     \displaystyle E_{\mathrm{GNMDS}}\left(\mathbf{K}'', \mathcal{T}\right) = 
     \hspace{-10 pt}\displaystyle \sum_{(a,b,c) \in \mathcal{T}} \hspace{-10 pt}\max(0, d_{\mathbf{K}''}(\mathit{x}_a, \mathit{x}_b) - d_{\mathbf{K}''}(\mathit{x}_a, \mathit{x}_c) + 1)
\label{eq:GNMDS-AK-Error}
\end{equation}

%This choice for $E$ fits into our framework.  However, a more common formulation of optimizations in margin-based learning involves the introduction of slack variables:

%% \begin{equation}
%%   \begin{array}{rll}
%%     \displaystyle \min_{\mathbf{K}_0, \mathbfup{\mu}, \xi_{abc}} & \multicolumn{2}{l}{\hspace{-10 pt}\displaystyle \sum_{(a,b,c) \in \mathcal{T}} \xi_{abc} + \lambda_1\mathrm{trace}(\mathbf{K}_0) +  \lambda_2\|\mathbfup{\mu}\|_1}\\ [20 pt]
%%     \mathrm{s.t.} & \hspace{-1 pt} \forall_{\left(a,b,c\right) \in \mathcal{T}}\ \hspace{-2 pt} : & \hspace{-5 pt}d_{\mathbf{K}''}(\mathit{x}_a, \mathit{x}_c) \geq \\
%%     & & \hspace{-5 pt} d_{\mathbf{K}''}(\mathit{x}_a, \mathit{x}_b) + (1 - \xi_{abc}), \\
%%     & \multicolumn{2}{l}{\mathbf{K}_0 \succeq 0,\ \mathbfup{\mu} \geq 0,\ \xi_{abc} \geq 0}
%%   \end{array}
%%   \label{eq:GNMDS-AK}
%% \end{equation}

We call our method with this error function GNMDS-AK.  The hinge loss ensures that only triplets that are unsatisfied by a margin of one increase the objective.  GNMDS-AK is also a convex optimization problem, due to Prop. \ref{prop:convexRCKL-AK} and the following:
\vspace{6pt}
\begin{proposition}
  \label{lemma:convexEGNMDS}
  \eqref{eq:GNMDS-AK-Error} is convex in both $\mathbf{K}_0$ and $\mathbfup{\mu}$.
\end{proposition}
%
%% \begin{convexGNMDS-AK}
%% \label{prop:convexRCKL-AK}
%%    \eqref{eq:GNMDS-AK} is a convex optimization problem.
%% \end{convexGNMDS-AK}
%% \begin{proof}
%% The sum over the slack variables and the positivity constraints are both convex \cite{boyd2004convex}.  With this fact and by lemmas \ref{prop:convexRCKL-AK} and \ref{lemma:convexEGNMDS} we have established  that all terms in the objective and all constraints are convex.  The sum of the three convex terms in the objective result in a convex function.  Thus, the objective function and the constraints are convex in the optimization variables, making \eqref{eq:STE} a convex optimization problem.
%% \end{proof}

%The error function attempts to minimize the non-negative slack variables,  $\mathrm{\xi}_{abc}$, that measure how much the margin constraints are violated.  Without the slack variables, the above formulation would not have a feasible solution in the presence of conflicting triplets.

Proposition \ref{lemma:convexEGNMDS} is proven in Sec. \ref{sec:appProposition3} of the appendix.  For a more rigorous comparison of RCKL methods see \cite{van2012stochastic}.  We propose to solve both STE-AK and GNMDS-AK via projected gradient descent methods; both of which are outlined in Sec. \ref{sec:appAKA} of the appendix.

%The main difference between STE and GNMDS is that in the case of GNMDS, once a triplet is satisfied with a margin of one, it no longer effects the objective.  In STE, all triplets contribute to the objective, though their contribution decays as their probabilities approach one.  
\section{Related Work}
\label{sec:related}
RCKL-AK can be viewed as a combination of \emph{multiple kernel learning} (MKL) and \emph{non-metric multidimensional scaling} (NMDS).  Learning a non-negative sum of kernels, as in \eqref{eq:mklcomb}, appears often in MKL literature, which is focused on finding efficient methods for learning a combination of predefined kernels for a learning task.  The most widely studied problem in MKL has been Support Vector Classification \cite{lanckriet2004learning,rakotomamonjy2008simplemkl,varma2009more,Jain12}.  To our knowledge there has been no application of MKL techniques to the task of learning a kernel from relative comparisons.

The RCKL problem posed in Sec. ~\ref{sec:pre} is a special case of the NMDS problem first formalized in \cite{agarwal2007generalized}, which in turn is a generalization of the Shepard-Kruskal NMDS problem \cite{shepard1962analysis}.  GNMDS, STE, and Crowd Kernel Learning (CKL) \cite{tamuz2011adaptively} are all methods that can be applied to the RCKL problem.  However, none of these methods consider inputs beyond relative comparisons.  Our work creates a novel RCKL method that uses ideas popularized in MKL research to incorporate side information into the learning problem.

Relative comparisons have also been considered in \emph{metric learning} \cite{schultz2004learning,davis2007information,huang2011generalized}.  In metric learning the focus is on learning a distance metric over objects that can be applied to out-of-sample \emph{objects}.  This work focuses specifically on finding a kernel over given objects that generalizes well to out-of-sample (unobtained) \emph{triplets}.  In this way, the goal of metric learning methods is somewhat different than the one in this work.  Two recent works propose methods to learn a Mahalanobis distance metric with multiple kernels: Metric Learning with Multiple Kernels (ML-MKL) \cite{wang2011metric} and Multiple Kernel Partial Order Embedding (MKPOE) \cite{mcfee2011learning}; the latter focusing exclusively on relative distance constraints similar to those in this work.  The kernel learned by RCKL-AK induces a mapping that is fundamentally different than those learned by these metric learning techniques.  Consider the mapping induced by one of the metric learning methods proposed in both \cite{mcfee2011learning} (Section 4.2) and \cite{wang2011metric} (Equation 6):

\begin{equation}
\mathbfup{\Phi}_{\mathbfup{\mu}, \mathbfup{\Omega}}(x) = \mathbfup{\Omega}\left[\sqrt{\mu_1}\mathbfup{\Phi}_1(x),...,\sqrt{\mu_A}\mathbfup{\Phi}_A(x)\right]
\label{eq:metricMap}
\end{equation} 

The derivation of this mapping can be found in Sec. \ref{sec:appDerivation} of the appendix. Here $\mathbfup{\Omega} \in \mathbb{R}^{m\mathrm{x}D}$ produces a new feature space by transforming the feature spaces induced by the auxiliary kernels.  Without $\mathbfup{\Omega}$, \eqref{eq:metricMap} learns a mapping similar to \eqref{eq:mklmap}.  The matrix $\mathbfup{\Omega}$ plays a role similar to the one $\mathbf{K}_0$ plays in RCKL-AK: it is learned to satisfy triplets that the auxiliary kernels alone cannot.  Instead of linearly transforming the auxiliary kernel feature spaces,  RCKL-AK implicitly learns new features that are concatenated onto the concatenated auxiliary kernel feature spaces (see \eqref{eq:rcklmap}).
 
In both works, the authors propose non-convex optimizations to solve for their metrics, and, in addition, different convex relaxations.  A critical issue with the convex solutions is that they employ SDPs over  $n^2*A$ (MKPOE-Full) and $n^2*A^2$ (NR-ML-MKL)  optimization variables, respectively.  For moderately sized problems these methods are impractical.  To resolve this issue, \cite{mcfee2011learning} propose a method that imposes further diagonal structure on the learned metric, reducing the number of optimization variables to $n*A$ (MKPOE-Diag), but in the process, greatly limit the structure of the metric.  RCKL-AK is a convex SDP with $n^2 + A$ optimization variables that does not impose strict structure on the learned kernel.  Unfortunately, by learning the unique kernel $\mathbf{K}_0$ directly and not the mapping $\mathbfup{\Phi}_0$ or a generating function of $\mathbf{K}''$, our method cannot be applied to out-of-sample objects.  Data analysis that does not require the addition of out-of-sample objects can be used over that kernel.  There are many unsupervised and semi-supervised techniques that fit this use case.

\section{Experiments}
\label{sec:experiments}
In order to show that RCKL-AK can learn kernels from few triplets that generalize well to unobtained triplets, we perform two experiments: one using synthetic data, and one using real-world data.  In both experiments, we compare STE and GNMDS variants of RCKL-T, RCKL-MKL, and RCKL-AK, as well as non-convex and convex variants of MKPOE.  For the MKPOE methods, we consider a triplet $\left(a,b,c\right)$ to be satisfied if $d_{\mathbf{M}}\left(x_a,x_b\right) < d_{\mathbf{M}}\left(x_a,x_c\right)$, where $d_{\mathbf{M}}$ is the distance function defined by the metric.   The STE and GNMDS implementations used are from \cite{van2012stochastic}, which are made publicly available on the authors' websites.  The MKL and AK versions were extended from these implementations.  MKPOE implementations were provided to us by their original authors.  All auxiliary kernels are normalized to unit trace, and all hyperparameters were validated via line or grid search using validation sets.

\subsection{Synthetic Data}
To generate synthetic data we begun by randomly generating 100 points in seven, independent, two-dimensional feature spaces where both dimensions were over the interval $[0,1]$.  Then, we created seven linear kernels, $\mathbf{K}_0,...,\mathbf{K}_6$ from these seven spaces.  We combined four of the seven kernels:
\begin{equation}
\mathbf{K}^{*} = \frac{1}{2}\mathbf{K}_0 + \frac{1}{4}\mathbf{K}_1 + \frac{1}{6}\mathbf{K}_2 + \frac{1}{12}\mathbf{K}_3 
\label{eq:gtkernel}
\end{equation}
We then used $\mathbf{K}^{*}$ as the ground truth to answer all possible, non-redundant triplets.  Following the experimental setup in \cite{tamuz2011adaptively}, we divided these triplets into 100 triplet ``rounds''.  A round is a set of triplets where each object appears once as the head $a$ being compared to randomly chosen objects $b$ and $c$.  From the pool of rounds, 20 were chosen to be the training set, 10 were chosen to be the validation set, and the remaining rounds were the test set.  This was repeated ten times to create ten different trials.
 
Next, we took all seven feature spaces and perturbed each point with randomly generated Gaussian noise.  From these new spaces we created seven new linear kernels $\hat{\mathbf{K}}_0,...,\hat{\mathbf{K}}_6$, of which $\hat{\mathbf{K}}_1,...,\hat{\mathbf{K}}_6$ were used as the input auxiliary kernels in the experiment.  Here, $\hat{\mathbf{K}}_1,...,\hat{\mathbf{K}}_3$ are kernels that represent attributes that influence how the ground truth makes comparisons between objects.  $\hat{\mathbf{K}}_4,...,\hat{\mathbf{K}}_6$ contain information that is not considered when making comparisons, and $\mathbf{K}_0$ represents intuition about the objects that was not or cannot be input as an auxiliary kernel.  

%Consider the example of a human comparing clothing items.  She may compare them by characteristics such as color, size, and material, but not traits like designer or pattern.  In addition, the human may have some personal fashion sense that she applies when comparing clothing items, which cannot be input as an auxiliary kernel.

We wish to evaluate the performance of each method as the number of triplets increases.  With more triplets each method should be able to build models that satisfy more unobtained triplets.  To show this we performed the following experiment.  For each trial, the 20 training rounds and 10 validation rounds are divided into ten subsets, each containing two training rounds and one validation round.  Starting with one of the subsets, each model is trained, setting the hyperparameters through cross-validation on the validation set, and evaluated on the test set.  Then, another subset is added to the training and validation sets.  We repeat this process until all ten subsets are included.  We evaluate the methods by the total number of unsatisfied triplets in the test set divided by the total number of triplets in the test set (test error).  Here, the test set represents unobtained triplets.  For all of the following figures, error bars represent a 95\% confidence interval.

{\bf Discussion:} Figure \ref{fig:Toy} shows the mean test error over the ten trials  as a function of the number of triplets in the training set.  Both RCKL-MKL methods improve performance initially, but achieve their approximate peak performance early and fail to improve as triplets are added.  This supports the claim that RCKL-MKL is limited by only being able to combine auxiliary kernels through a conic combination.  Both RCKL-T methods perform much worse than either of the methods that use auxiliary kernels.  Without the side information provided by the auxiliary kernels, RCKL-T cannot generalize to test triplets with few training triplets.

\begin{figure}
  \centering
  \null\hfill
  \begin{subfigure}{0.48\columnwidth}
    \centering
    \includegraphics[width=\columnwidth]{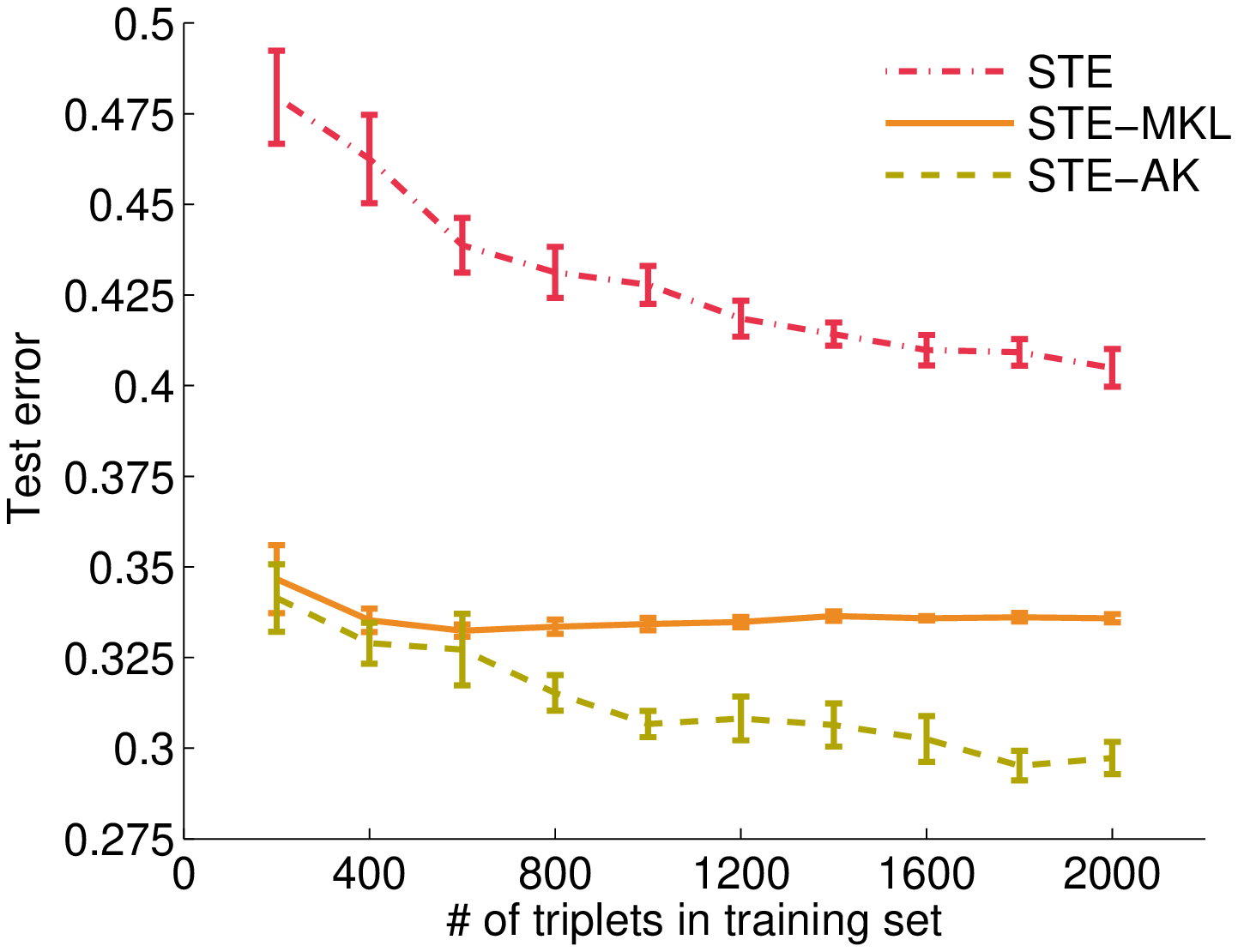}
    \label{fig:STEerrorToy}
  \end{subfigure}%
  \hfill %
  \begin{subfigure}{0.48\columnwidth}
    \centering
    \includegraphics[width=\columnwidth]{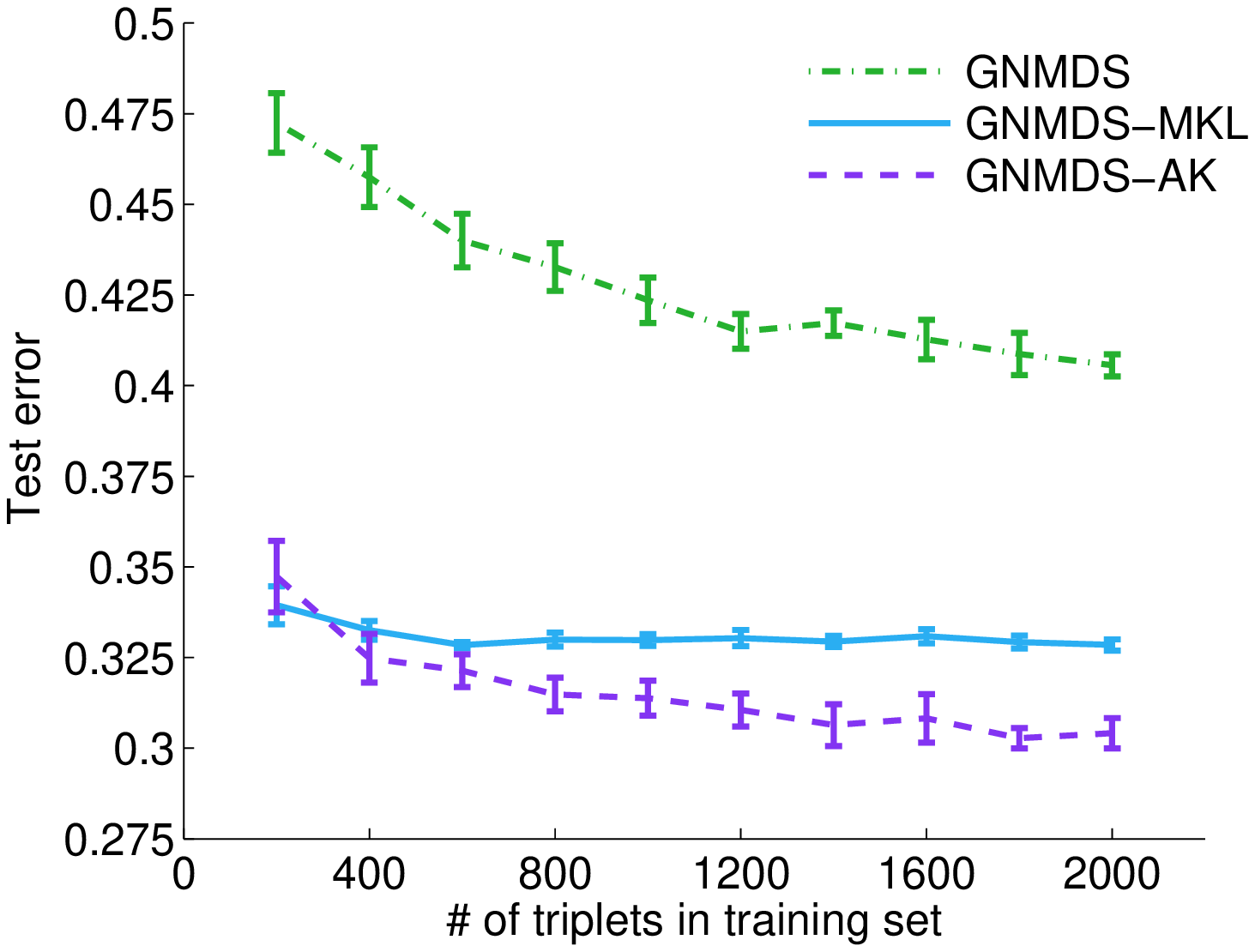}
    \label{fig:GNMDSerrorToy}
  \end{subfigure}
  \hfill\null
    \caption{Mean test error over ten trials of the synthetic data set}
    \label{fig:Toy}
\end{figure}

We believe this experiment demonstrates the utility of both $\mathbf{K}_0$ and the auxiliary kernels in RCKL-AK.  With very few training triplets, the RCKL-AK methods relied on  the auxiliary kernels, thus the performance is similar to the RCKL-MKL methods.  As triplets are added, the RCKL-AK methods used $\mathbf{K}_0$ to satisfy the triplets that a conic combination of the auxiliary kernels could not.  Further evidence for this is shown by the fact that the rank of $\mathbf{K}_0$ increased as the number of training triplets increased.  For example, for STE-AK, the mean rank of $\mathbf{K}_0$ was 85.6, 94.2, and 96.2 for 200, 400, and 600 triplets in the training set, respectively.  In other words, the optimal settings of $\lambda_1$ and $\lambda_2$ made $\mathbf{K}_0$ more expressive as the number of triplets increased.      

Ideally, the RCKL-AK methods should eliminate $\hat{\mathbf{K}}_4$, $\hat{\mathbf{K}}_5$, and $\hat{\mathbf{K}}_6$ from the model by reducing their corresponding weights $\mu_4$, $\mu_5$, and $\mu_6$ to exactly zero.   Figure \ref{fig:ToyMus} shows the values of the $\mathbfup{\mu}$ parameter for STE-AK and GNMDS-AK as the number of triplets increase.  Both RCKL-AK methods correctly identify the three auxiliary kernels from which the ground truth kernel was created by setting their corresponding weight parameters to be non-zero.  In addition, they assigned weights to the kernels roughly proportional to the ground truth.  The three noise kernels were assigned very low, and often zero weights.  The RCKL-MKL methods learned similar values for the elements of $\mathbfup{\mu}$ than those in Fig. \ref{fig:ToyMus}.  Since RCKL-MKL learned the relative importance of the auxiliary kernels with only few triplets, it had achieved approximately its peak performance and could not improve further with the addition of more triplets. 

\begin{figure}
        \centering
        \null\hfill
        \begin{subfigure}{0.48\columnwidth}
                \centering
                \includegraphics[width=\columnwidth]{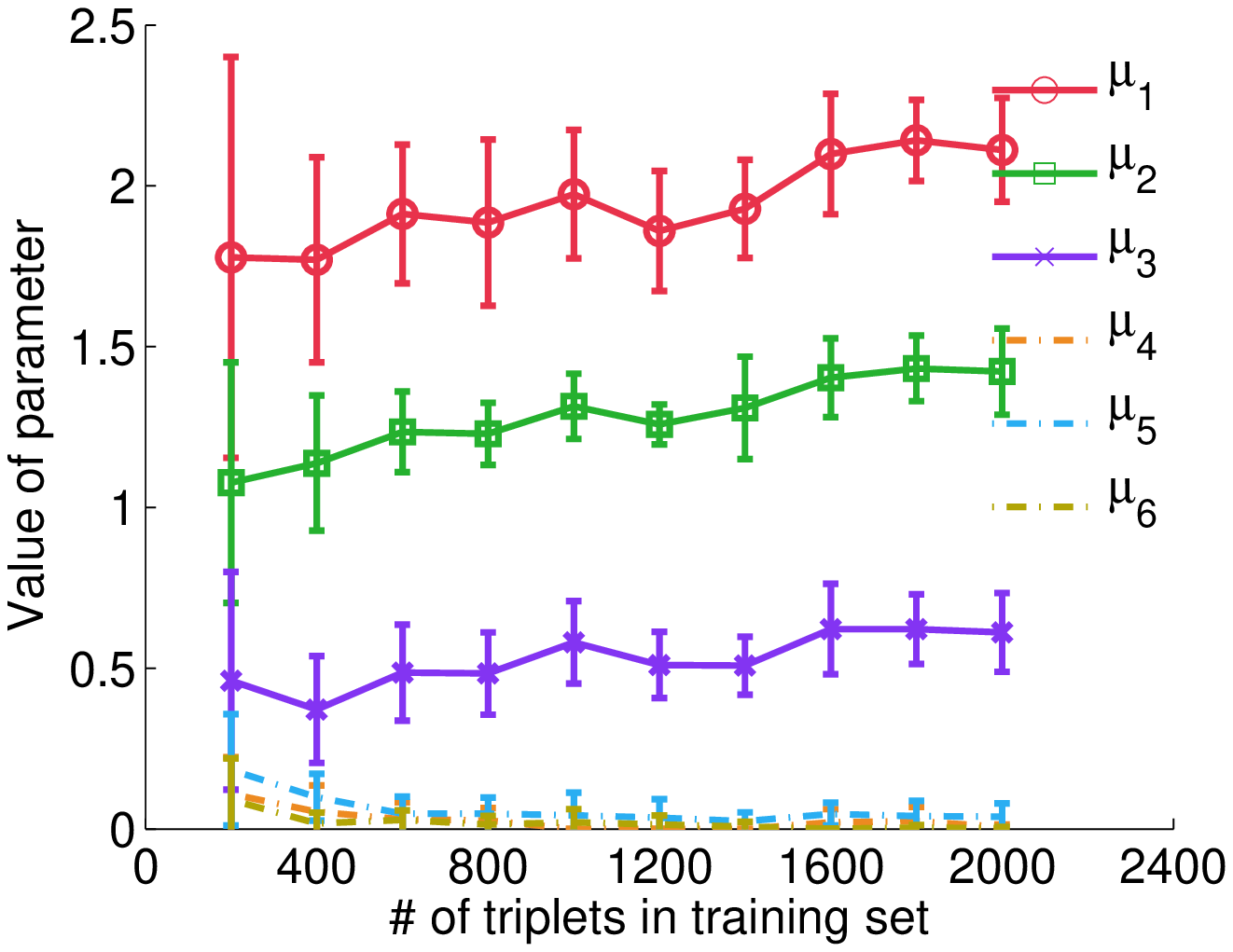}
                \caption{STE-based models}
                \label{fig:STEmusToy}
        \end{subfigure}
        \hfill
        \begin{subfigure}{0.48\columnwidth}
                \centering
                \includegraphics[width=\columnwidth]{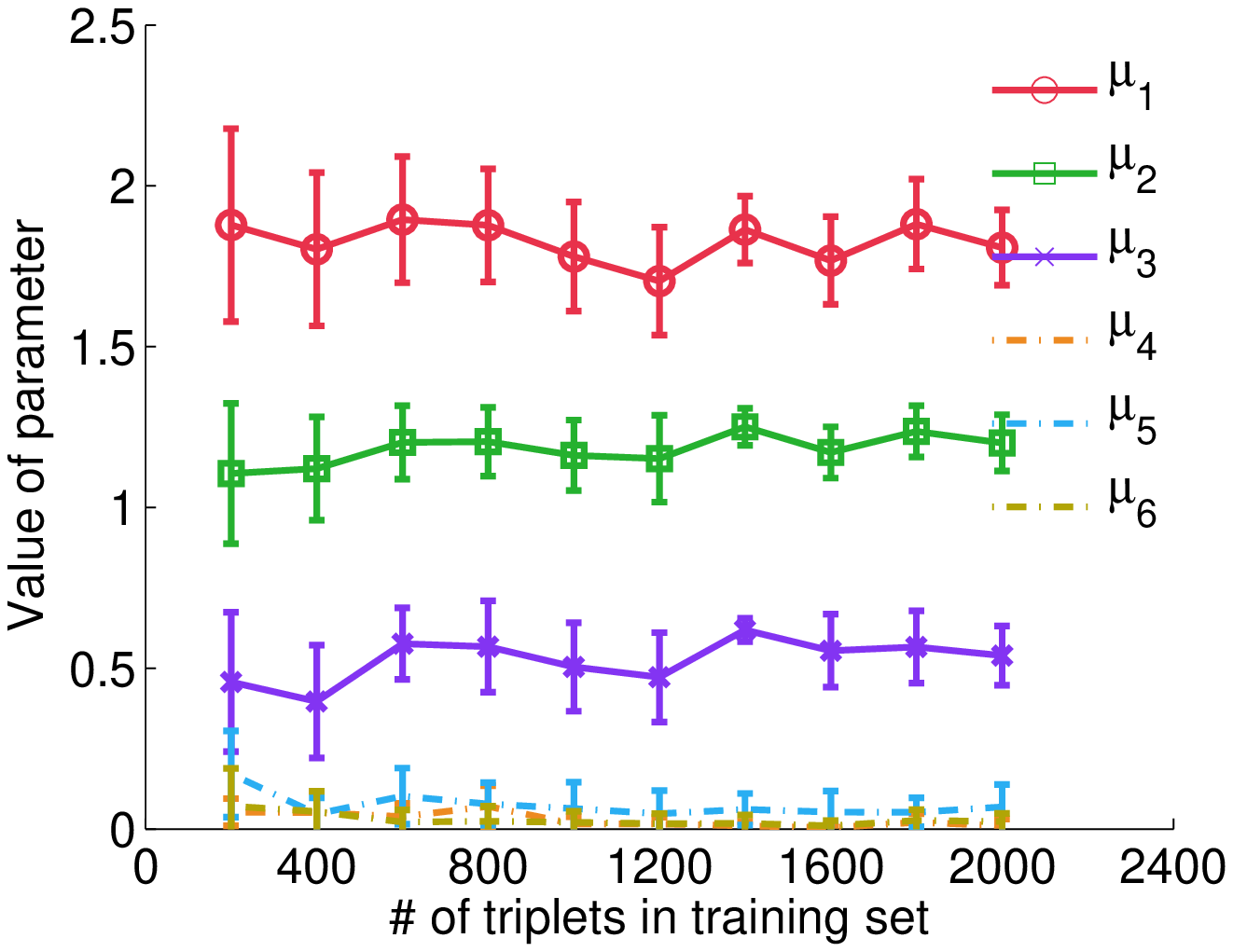}
                \caption{GNMDS-based models}
                \label{fig:GNMDSmusToy}
        \end{subfigure}
        \hfill\null
        \caption{Mean values of $\mathbfup{\mu}$ on synthetic data}\label{fig:ToyMus}
\end{figure}

Figure \ref{fig:ToyMKPOE} shows the same STE-AK, GNMDS-AK, and GNMDS-MKL error plots as Fig. \ref{fig:Toy}, but also includes three variations of MKPOE: A non-convex formulation (MKPOE-NC), and two convex formulations (MKPOE-Full and MKPOE-Diag).  All metric learning methods perform very similarly, yet worse than RCKL-MKL and RCKL-AK.  We believe that  the MKPOE methods must transform the auxiliary kernel space drastically to satisfy the few triplets. By doing this they lose much of the information in the auxiliary kernels that allows RCKL-MKL and RCKL-AK methods to form more general solutions. 

\begin{figure}
        \centering
        \null\hfill
        \begin{subfigure}{0.48\columnwidth}
                \centering
                \includegraphics[width=\columnwidth]{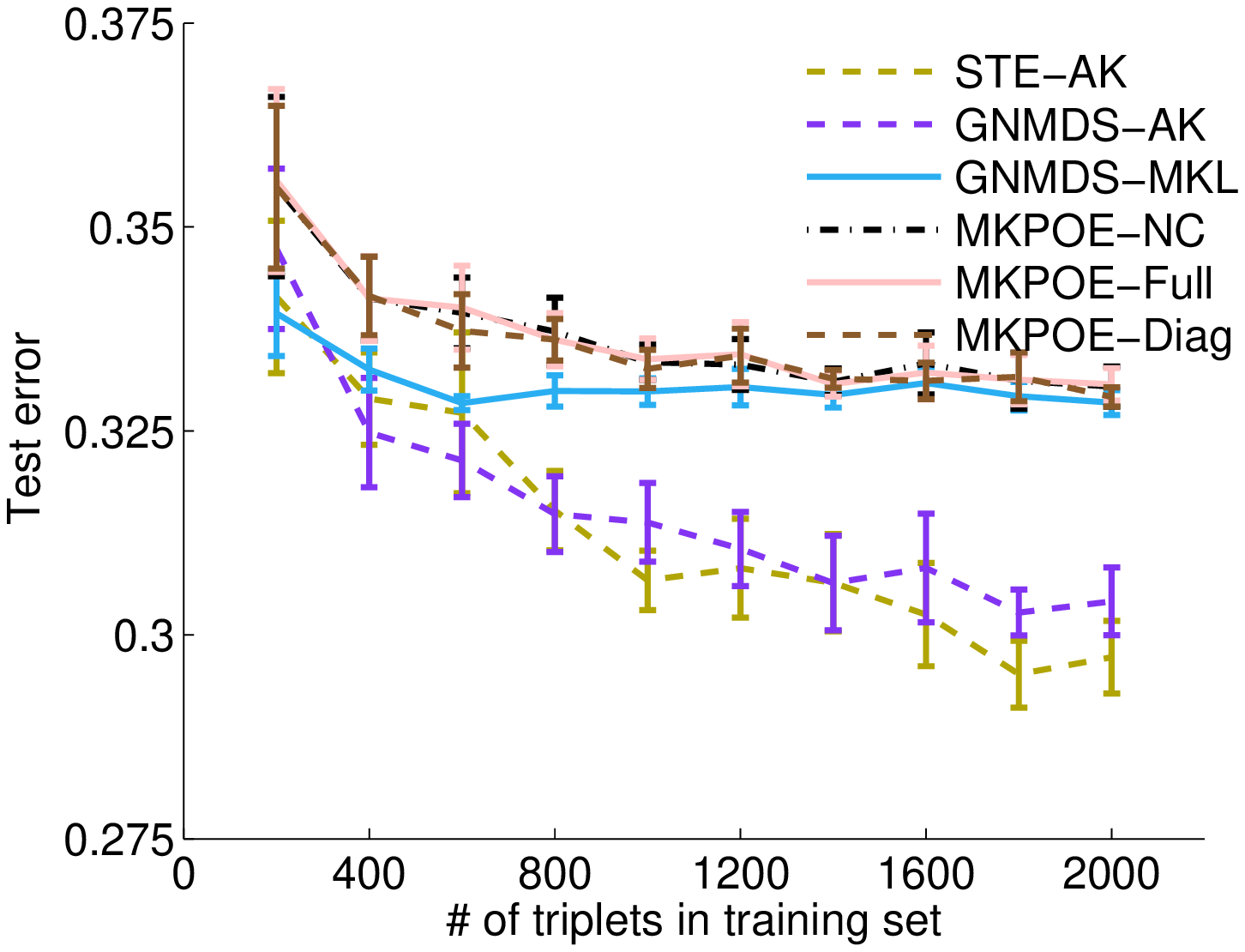}
                \caption{Synthetic data set}
                \label{fig:ToyMKPOE}
        \end{subfigure}
        \hfill
        \begin{subfigure}{0.48\columnwidth}
                \centering
                \includegraphics[width=\columnwidth]{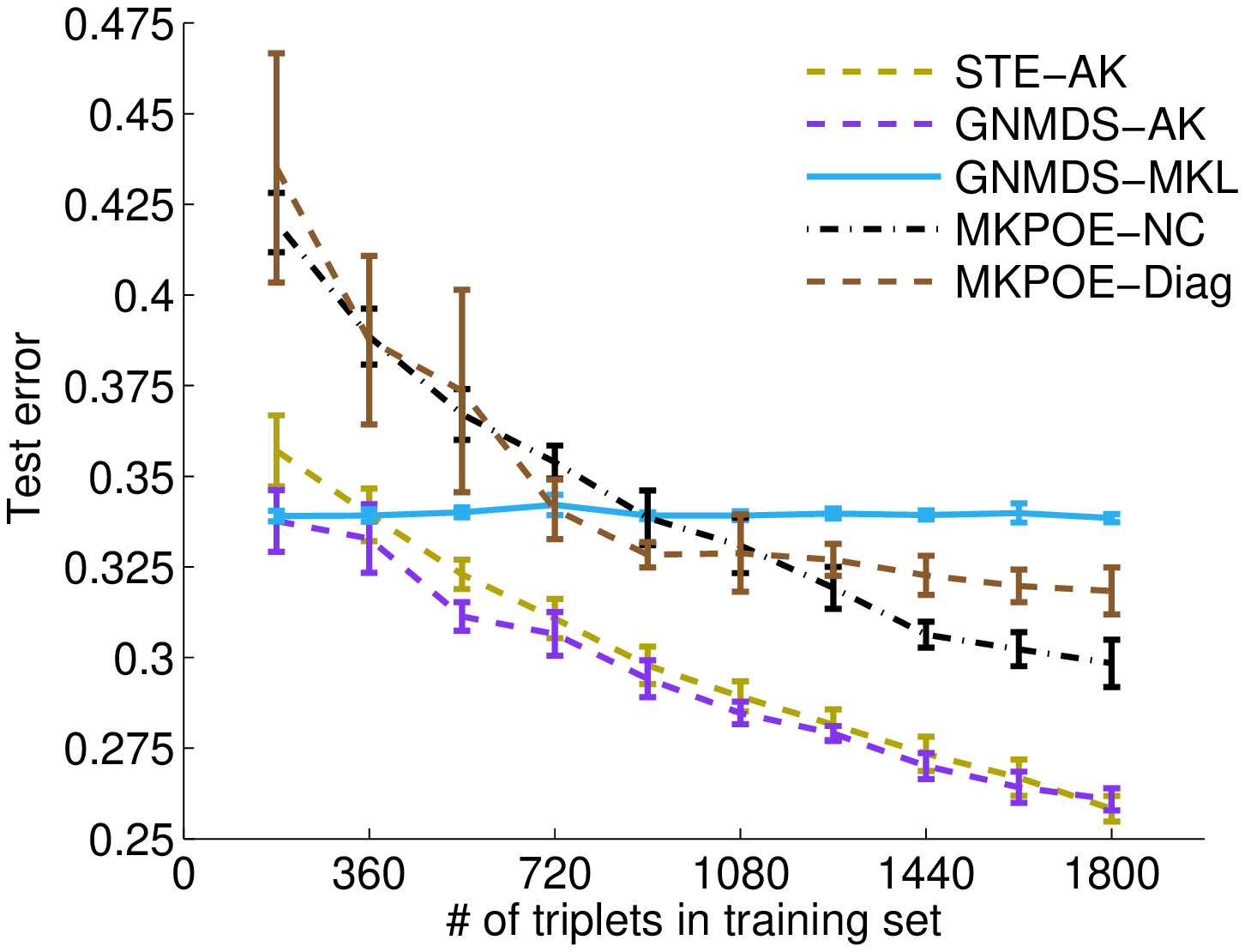}
                \caption{\emph{aset400} data set}
                \label{fig:MusicMKPOE}
        \end{subfigure}
        \hfill\null
        \caption{Mean test error over ten trials}
        \label{fig:MKPOE}
\end{figure}

\subsection{Music Artist Data}
We also performed an experiment using comparisons among popular music artists.  The \emph{aset400} dataset \cite{ellis2002quest} contains 16,385 relative comparisons of 412 music artists gathered from a web survey, and \cite{mcfee2009heterogeneous} provides five kernels built from various features describing each artist and their music.  Two of the kernels were built from text descriptions of the artists, and three were built by extracting acoustic features from songs by each artist.

The \emph{aset400} dataset provides a challenge absent in the synthetic data: not all artists appear in the same number of triplets.  In fact, some artists never appear as the head of a triplet at all.  As a result, this dataset represents a setting where feedback was gathered non-uniformly amongst the objects.  In light of this, instead of training the models in rounds of triplets, we randomly chose 2000 triplets as the development set; the rest were used as the test set.  Like before, we broke the development set into ten subsets, and progressively added subsets to the working set, training and testing each iteration.  Ten percent of the working set was used for validation and 90 percent was used for training.  The experiment was performed ten times on different randomly chosen train/validation/test splits.

{\bf Discussion:} The results, shown in Fig, \ref{fig:Artist}, are similar to those for the synthetic data with a few key differences.  The RCKL-MKL methods did not perform as well relative to the RCKL-T methods.  This could be attributed to the fact that the auxiliary kernels here did not reflect the triplets as well as those in the synthetic data.  Only one kernel was consistently used in every iteration (the kernel built from artist tags).  The rest were either given little weight or completely removed from the model.  As with the synthetic data, with 200 and 400 training triplets the RCKL-AK methods performed as well as their respective RCKL-MKL counterparts, but as more triplets were added to the training set, the RCKL-AK methods began to perform much better.  In this experiment, the RCKL-T methods became more competitive, but were outperformed significantly by RCKL-AK much of the time.  This, again, could be because the auxiliary kernels were less useful than with the synthetic data.  

\begin{figure}
        \centering
        \null\hfill
        \begin{subfigure}{0.48\columnwidth}
                \centering
                \includegraphics[width=\columnwidth]{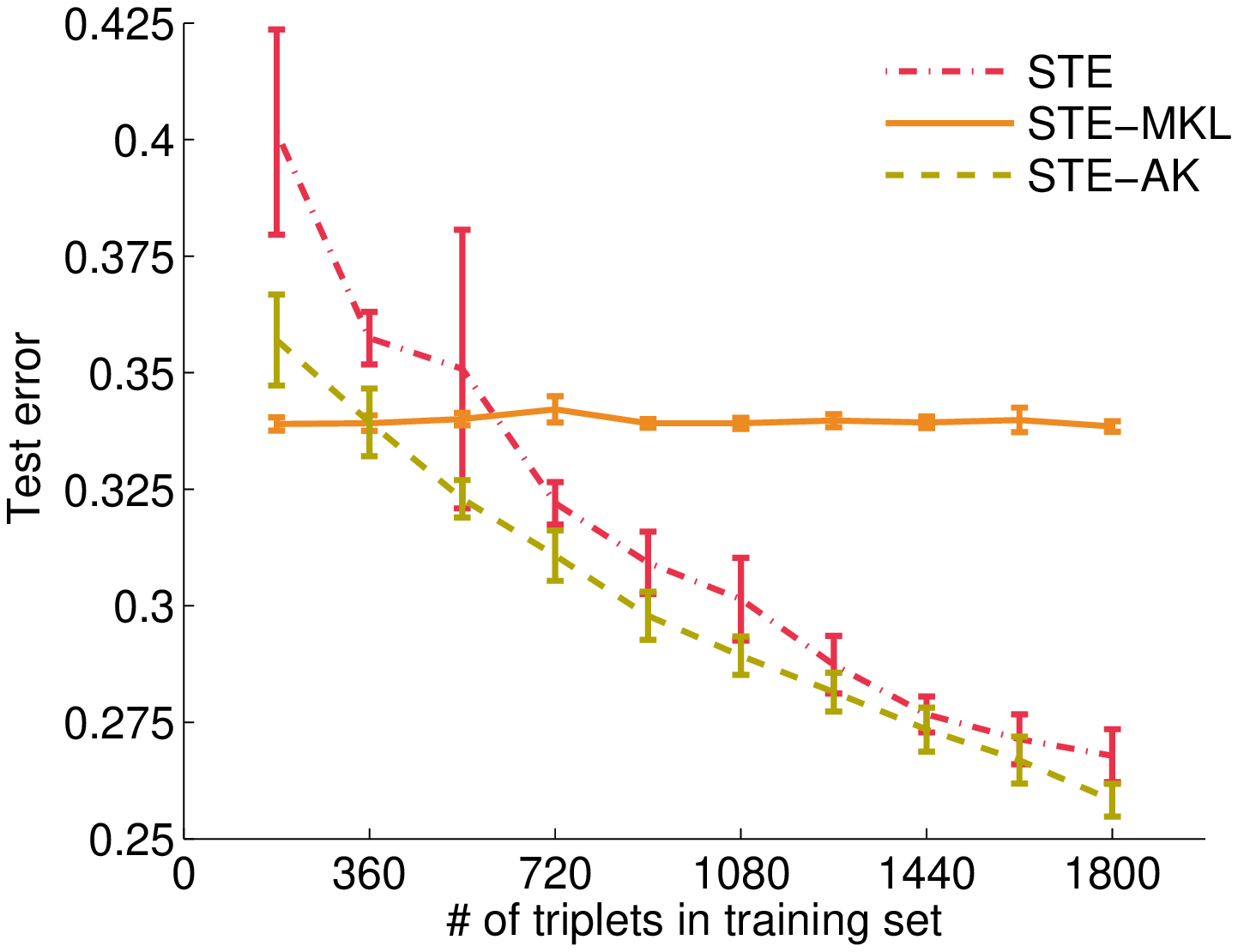}
                \label{fig:STEerrorArtist}
        \end{subfigure}
        \hfill
        %\begin{subfigure}[b]{0.5\textwidth}
        \begin{subfigure}{0.48\columnwidth}
                \centering
                \includegraphics[width=\columnwidth]{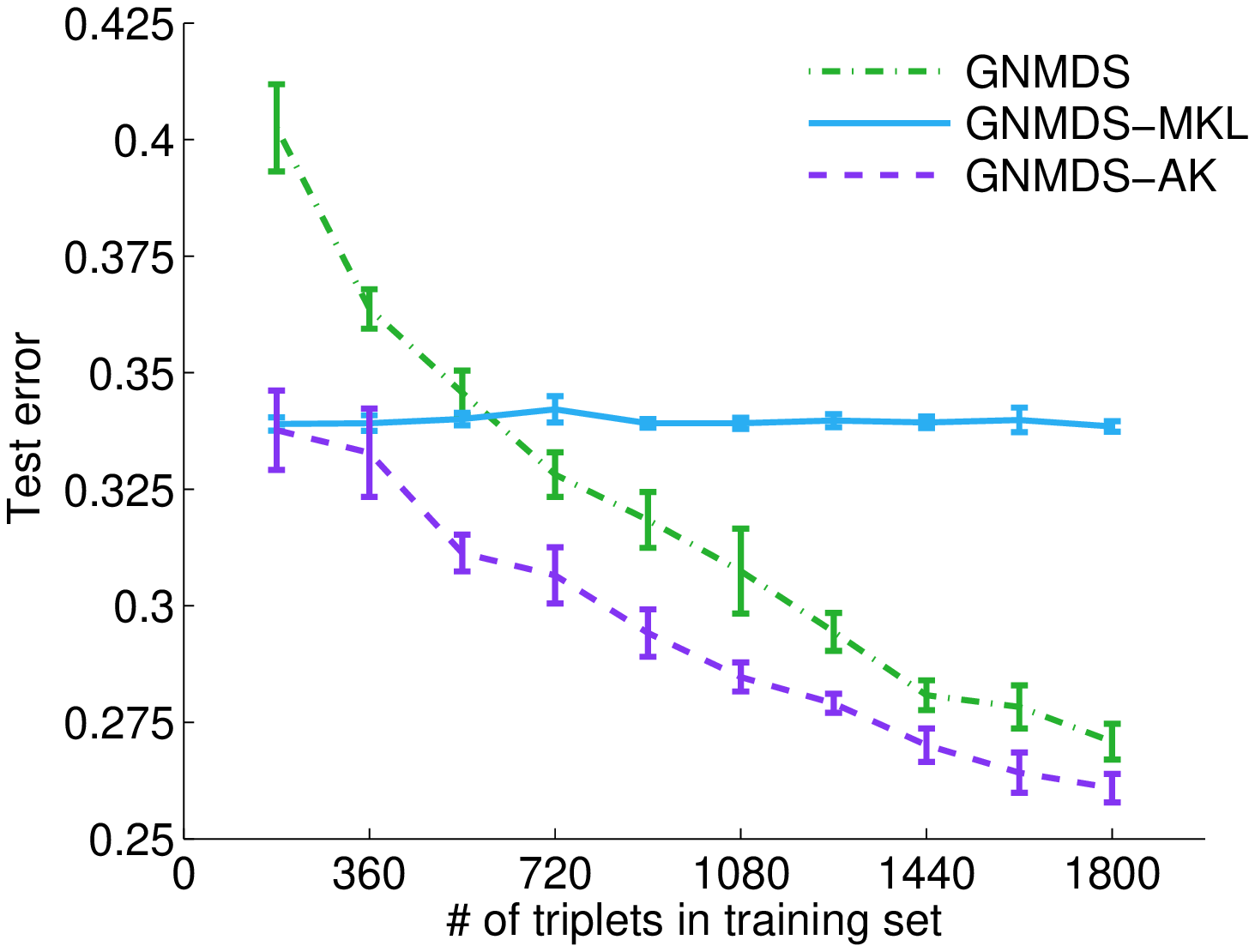}
                \label{fig:GNMDSerrorArtist}
        \end{subfigure}
        \hfill\null
        \caption{Mean test error over ten trials of the \emph{aset400} data set}
        \label{fig:Artist}
\end{figure}

Figure \ref{fig:MusicMKPOE} compares the performance of MKPOE-NC and MKPOE-Diag on the \emph{aset400} dataset to the RCKL-AK methods as well as GNMDS-MKL.  MKPOE-Full could not be included in this experiment due to its impractically long run-time for an experiment of this size.  Both MKPOE methods perform similarly, and seem to suffer greatly from the lack of meaningful auxiliary kernels, but do improve as the number of triplets increases.  Still, over all experiments, the MKPOE methods have statistically significantly higher test error than both RCKL-AK methods.  

%The trend seems to indicate that eventually the RCKL and RCKL-AK methods would converge to the same performance.  If the auxiliary kernels became a hindrance, RCKL-AK has the ability to regularize the $\mathbfup{\mu}$ parameters to the point where they all would be set to zero, thus becoming equivalent to RCKL-T.

\section{Conclusions and Future Work}
\label{sec:future}
In this work we propose a method for learning a kernel from relative comparisons called Relative Comparison Kernel Learning with Auxiliary Kernels (RCKL-AK) that supplements given relative comparisons with auxiliary information.  RCKL-AK is a convex SDP that can be solved by adding slight computational overhead to traditional methods and more efficiently than many metric learning alternatives.  Experimentally, we show that RCKL-AK learns kernels that generalize to more out of sample relative comparisons than the aforementioned traditional and metric learning methods.

There are three main directions of future work.  First, we wish to create a more efficient formulation of RCKL-AK.  While common in solving SDPs, the most time-consuming step in RCKL-AK is performing eigendecomposition in order to project the kernel onto the PSD cone after each gradient step.  In \cite{kulis2006learning} the authors create a kernel learning method that eliminates the need to project onto the PSD cone.  We will investigate extending their method to use auxiliary kernels.  Second, we would like to extend our method to out-of-sample objects.  In \cite{bengio2004out}, generating functions are learned for kernels that were learnt through various popular techniques (LLE, Isomap, etc.).  Similarly, we will attempt to learn a generating function for our kernel.  Finally, we will explore practical applications of our method, specifically, the use of RCKL-AK for product recommendation.

\appendix
\section{Proof of Theorems}
\subsection{Theorem 1}
\label{sec:appTheorem1}
Below we prove Thm. \ref{thm:totalNumTriplets} from Section \ref{sec:theory}
\begin{proof}
For $n$ objects there are $\binom{n}{2}$ pair-wise distances between objects (we do not consider distances between objects and themselves).  A triplet is a comparison between two pair-wise distances with a common object.  Consider a single pairwise distance, and without loss of generality, let this distance be $d(x_a, x_b)$.  There exists exactly $n-2$ other pairwise distances that contain $x_a$  (each with a different second object that is not $x_a$ or $x_b$), and exactly $n-2$ other pairwise distances that contain $x_b$.  Thus, each pair-wise distance can be compared to $2n-4$ other distances.  As a result there are exactly, $\frac{1}{2}\binom{n}{2}*(2n-4) = \frac{1}{2}(n^3-3n^2+2n)$ triplets (the $\frac{1}{2}$ comes from the fact that a single triplet $(a,b,c)$ counts as an answer to the relative comparison ``is $x_a$ more similar to $x_b$ than $x_c$'' and ``is $x_a$ more similar to $x_c$ than $x_b$'').
\end{proof}

\subsection{Theorem 2}
\label{sec:appTheorem2}
To prove Thm. \ref{thm:numTriplets} from Section \ref{sec:theory} we use the directed acyclic graph representation of relative comparisons given in Section 2 of \cite{mcfee2011learning}.  A set of triplets $\mathcal{T}$ is represented by a graph $\mathcal{G} = (\mathcal{V},\mathcal{E})$.  Here, a vertex represents a distance between two objects, and a an edge represents a relative comparison.  Let $v_{\{1,2\}} \in \mathcal{V}$ be a vertex representing the distance between objects $x_1$ and $x_2$ indexed by the unordered pair $\{1,2\}$.  Note the slight notational change from the main body of the paper: Instead of indexing objects by alphabetic characters, we choose to index them by natural numbers for convenience.  As an example, the triplet $(1,2,3)$ can be represented by a directed edge $e = (v_{\{1,2\}} \to v_{\{1,3\}}) \in \mathcal{E}$ from vertex $v_{\{1,2\}}$ to vertex $v_{\{1,3\}}$.  Because this work considers triplets, vertices need to have a common object between them for there to be a edge between them.  As an example, vertices $v_{\{1,2\}}$ and $v_{\{1,3\}}$ can have an edge between them, but $v_{\{1,2\}}$ and $v_{\{3,4\}}$ cannot.  Any cycle in $\mathcal{G}$ constitutes a conflict in relative comparisons.  For this work we assume that there exists no conflicts or that conflicts can be resolved either algorithmically, such as in \cite{mcfee2011learning} or by querying the source of triplets again.  From this point forward we say an edge is ``valid'' if it does not create a conflict and it connects two vertices that have a common object between them.  An edge, and thus the corresponding triplet, can be inferred if there exists a path of length greater than 1.  For example, if there is an edge from $v_{\{1,2\}}$ to $v_{\{1,3\}}$, and an edge from $v_{\{1,3\}}$ to $v_{\{1,4\}}$ then the edge from $v_{\{1,2\}}$ to $v_{\{1,4\}}$ can be inferred.

Theorem 2 considers the case where triplets are given one at a time.  Let $t_i$, be the $i$th triplet to be given, and $\mathcal{T}_i = \{t_j | j \leq i\}$.  Let $\mathcal{T}^{trans}_i$ be the triplets that can be inferred by $\mathcal{T}_i$.  For this analysis we study the directed acyclic graph analogues of these sets of triplets: $\mathcal{G}_i$ and $\mathcal{G}^{trans}_i$.  Let $e_i \in \mathcal{E}_i$ be the edge given at time $i$. We assume that an edge given at time $i$ has not been given at time $j < i$.  Next, we introduce the following recursive adversarial strategy for giving edges to this graph (square brackets ([]) to denote an ordered list of elements with subscripts indicating each element's position in the list):

\begin{algorithm}
\center \caption{Adversarial Strategy}
\begin{algorithmic}[1]
\Function{Adversary}{$n \in \mathbb{N}_{> 2}$}
  \If{$n == 3$}
     \State $e_1 \gets (v_{\{1,3\}} \to v_{\{1,2\}})$ 
     \State $e_2 \gets (v_{\{2,3\}} \to v_{\{1,2\}})$ 
     \State $e_3 \gets (v_{\{2,3\}} \to v_{\{1,3\}})$
     \State $idx \gets 3$
  \Else
     \State $idx \gets \frac{1}{2}((n-1)^3-3(n-1)^2+2(n-1))$
     \State $[e_1,...,e_{idx}] \gets $\Call{Adversary}{$n-1$}
     \State $\mathcal{V}^{new} \gets $\Call{GetNewVertices}{$n$}
     \State $[v_1^{old},...,v_{\binom{n-1}{2}}^{old}] \gets $\Call{FindOldVertexOrder}{$[e_1,...,e_{idx}]$}
     \For{$j = 1$ to $\binom{n-1}{2}$}
        \State $[e_{idx+1},e_{idx+2}] \gets $\Call{AddRemainingEdges}{$v_j^{old}$,  $\mathcal{V}^{new}$}
        \State $idx \gets idx + 2$
     \EndFor
     \ForAll{vertices in $\mathcal{V}^{new}$}
         \State $v_{rnd} \gets$ \Call{SelectRandomVertex}{$\mathcal{V}^{new}$}
         \State $[e_{idx+1}$,...,$e_{idx'}], idx' \gets $\Call{AddRemainingEdges}{$v_{rnd}$, $\mathcal{V}^{new}$}
         \State $idx \gets idx'$
         \State $\mathcal{V}^{new} \gets \mathcal{V}^{new} \backslash v_{rnd}$
     \EndFor
  \EndIf
  \State \textbf{return} $[e_1,...,e_{idx}]$
\EndFunction
\end{algorithmic}
\label{alg:adversarialStrategy}
\end{algorithm}

The following lines describe the process in which edges are given by Alg. \ref{alg:adversarialStrategy} in more detail:

\begin{enumerate}
  \item \textbf{Lines 2-6:} If n=3 (the base case) give the three edges on lines 3-5.
  \item \textbf{Line 9:} Get the adversarial solution for n-1 objects.
  \item \textbf{Line 10:} Call the procedure ``GetNewVertices'' which returns the set of vertices in the graph for $n$ objects that are not in the graph for $n-1$ objects.  This set is all vertices with $n$ their indexing pair.  We will call these vertices ``new vertices'' and all edges that connect to at least one new vertex ``new edges''.  We will call the vertices in the solution that are not new vertices ``old vertices'', and the edges that connect two old vertices ``old edges''.  Note, that $e_1,...,e_{idx}$ on line 9 contains every old edge and no new edges.
  \item \textbf{Line 11:} Call the procedure ``FindOldVertexOrder'', which takes the ordered set of old edges given as input.  This procedure orders all old vertices in the order in which they appear in the incoming side of an edge.  If an old vertex has no incoming edges it is added to the end of the list.  For instance, if the input is the solution for the base case, the output would be the list $[v_{\{1,2\}},v_{\{1,3\}},v_{\{2,3\}}]$.
  \item \textbf{Line 12:} Loop for each old vertex.
  \item \textbf{Line 13:} Call the procedure ``AddRemainingEdges'' which exhaustively gives all remaining edges involving $v_i^{old}$ in any order and points them towards $v_i^{old}$, assigning them to $e_{idx+1}$ and $e_{idx+2}$.  After this, all valid edges to or from $v_i^{old}$ are given.  Note that these new edges given here point from new vertices to $v_i^{old}$.
   \item \textbf{Line 14:} The index for the given edge numbers is incremented.
   \item \textbf{Line 16:} Loop for each new vertex
   \item \textbf{Line 17:} Call procedure ``SelectRandomVertex'' which randomly selects a vertex from $\mathcal{V}^{new}$ and assigns it to $v_{rnd}$
   \item \textbf{Line 18:} Call the procedure ``AddRemainingEdges'' which gives all remaining valid edges $e_{idx+1}$,...,$e_{idx'}$ involving $v_{rnd}$ and have them point toward $v_{rnd}$.  Note that these edges only connect new vertices together.
   \item \textbf{Lines 19 and 20:} Update the current number of edges given and remove $v_{rnd}$ from $\mathcal{V}^{new}$ so it does not get chosen in line 21 next iteration.
   \item \textbf{Line 23:} Return the given edges, their order indicated by their indices.
\end{enumerate}  

Next, we state the following lemmas:

\begin{lemma}
  \label{lemma:allEdges}
  Algorithm \ref{alg:adversarialStrategy} exhaustively gives all valid edges for a given $n$.
\end{lemma}

\begin{proof}
The loop beginning on line 12 loops over all old vertices and exhaustively gives all edges involving these vertices.  The loop beginning at line 16 does the same for all new vertices.  Since the sets of old and new vertices together include every vertex, valid edges are given by line 23 and are returned.
\end{proof}

In the subsequent proofs we will refer to any valid edge $e \notin \mathcal{E}_i$ to be ``ungiven'' at time $i$.  A fact used to prove the following two lemmas is that when the adversarial strategy gives an edge at time $i$, a vertex or group of vertices can never be used to infer an ungiven edge if two conditions hold:

\begin{enumerate}
  \item If all valid edges have been given involving a vertex or a group of vertices at time $i-1$
  \item If there does not exist an edge leaving a vertex or group of vertices to a vertex with at least one ungiven edge at time $i-1$    
\end{enumerate}

The first condition is somewhat obvious in that if all edges are given involving a vertex or group of vertices, then no ungiven edges involving this vertex can be inferred, because there are no ungiven edges involving this vertex.  Thus, the only way for a vertex or group of vertices can be used to infer an ungiven edge is that a path through them is used to infer an ungiven edge between two external vertices.  The second condition ensures that no such path exists.  As a result, regardless of what edge the adversary gives at time $i$, the ungiven inferred edges of graph $\mathcal{G}_i$ is always the same as that on $\mathcal{G}_i$ with all vertices or groups of vertices removed that satisfy these two conditions. With this we can prove the following two lemmas:  

\begin{lemma}
\label{lemma:neverInfer3}
Line 13 of Alg. \ref{alg:adversarialStrategy} never gives edges that infer ungiven edges.
\end{lemma}
\begin{proof}
Consider line 11 of Alg. \ref{alg:adversarialStrategy}.  It orders the old vertices in the order in which they appear as the vertex being pointed to in the solution for $n$ objects.  Any vertex in this order only points to those that appear before it in the list, because when a vertex is chosen by this algorithm all valid edges involving this vertex are given pointing inward.  Consider the first iteration of the for loop starting on line 12.  The first vertex in this list has only incoming edges (namely, $v_{1,2}$ due to the base case).  Adding all remaining ungiven edges inward to this node cannot create any inferred edges because all old edges are going inward due to it being first in the list, thus Lem. \ref{lemma:neverInfer3} holds. In addition, after line 13 this node satisfies the first condition above.  It also satisfies the second condition as there exists no outgoing edges, thus no paths can go through it.  By these two conditions, it can be eliminated from the graph.  By doing so, all edges connecting it to other vertices are also eliminated.  As a result, the second vertex effectively becomes the first vertex in the list, making it only have incoming edges.  In the next iteration the second vertex is chosen, and, again, all remaining edges are pointed inward.  Adding of these edges also satisfies Lem. \ref{lemma:neverInfer3}, the same conditions hold, and it can effectively be removed.  The loop repeats this for all old vertices.  As such, for all iterations of the loop starting on line 12, line 13 never gives edges that infer ungiven edges. 
\end{proof}

\begin{lemma}
\label{lemma:neverInfer2}
Line 18 of Alg. \ref{alg:adversarialStrategy} never gives edges that infer ungiven edges.
\end{lemma}

\begin{proof}
By line 16 all edges connected to or from any old vertex has already been given via the loop staring on line 12.  Thus, the first condition above is satisfied for all old vertices.  Line 13 only adds edges from new vertices to old vertices, no edge leaves the subgraph of old vertices to the subgraph of new vertices.  This satisfies the second condition.  As a result, all old vertices can be eliminated from consideration.  What is left is just the new vertices with no edges connecting them to any other vertex.  Consider the first iteration of the loop starting on line 16.  Selecting a vertex $v_{rnd}$ and giving all edges pointing inward cannot infer any edges, because it has only incoming edges.  Thus, this operation does not violate Lem. \ref{lemma:neverInfer2}. In addition, the first condition is satisfied when all edges are given, and since there are no paths through this vertex, condition 2 is satisfied.  As a result, this vertex too can then be eliminated.  For the next iteration we now have one less new vertex, but, again, no edges connecting the ones left.  This loop repeats until all edges are given.    
\end{proof}

Given these three lemmas we proceed by proving the following proposition:
\begin{proposition}
For a given $n \in \mathbb{N}_{>2}$ , there exists an adversarial strategy for giving edges such that $\forall_{i \in\{1, ..., |\mathcal{E}^{total}|\}}$, $\mathcal{E}^{trans}_i \backslash \mathcal{E}_i = \emptyset$
\label{prop:neverInfer}
\end{proposition}
\begin{proof}
Let $\mathcal{E}_i = \{e_1,...,e_i\}$ where $\{e_1,...,e_i\}$ is the first $i$ triplets in the ordered set returned by Alg. \ref{alg:adversarialStrategy}.  First, we need to ensure $\mathcal{E}_i$ can be constructed for all $i \in\{1, ..., |\mathcal{E}^{total}|\}$.  Lemma \ref{lemma:allEdges} states that Alg. \ref{alg:adversarialStrategy} gives all edges for a given $n$, thus proving this.  Lastly, we need to show that Alg. \ref{alg:adversarialStrategy} never adds an edge to the sequence of edges that will infer an ungiven edge. To prove this, we use induction.  The base case is $n = 3$, which is the fewest number of objects for which triplets can be defined.  The base case is defined on lines 2-6 of Alg. \ref{alg:adversarialStrategy}.  The edges $e_1$ and $e_2$ do not infer the third and final edge $e_3$, thus Prop. \ref{prop:neverInfer} is true for the case $n = 3$.\\

Line 9 in Alg. \ref{alg:adversarialStrategy} returns the solution for $n-1$.  Thus, if Alg. \ref{alg:adversarialStrategy} never gives an edge that can infer an ungiven edge after this point in the algorithm, the inductive step is proven.  The only two lines after line 9 that add edges are lines 13 and 18.  By Lem. \ref{lemma:neverInfer3}, line 13 never adds an edge that can infer an ungiven edge and by Lem. \ref{lemma:neverInfer2} neither does line 18.  As a result the inductive step and thus Prop. \ref{prop:neverInfer} is proven. 
\end{proof}

Finally, with these propositions we can prove Thm. \ref{thm:numTriplets} from Section \ref{sec:theory}:
\begin{proof}

By Prop. \ref{prop:neverInfer}, there exists an adversarial strategy such that, for all $i \in \{1,...,|\mathcal{E}^{total}|\}$, $\mathcal{E}^{trans}_i \backslash \mathcal{E}_i = \emptyset$.  By construction of the directed acyclic graphs representing sets of triplets we can construct sets $\mathcal{T}_i$ from $\mathcal{E}_i$, $\mathcal{T}^{trans}_i$ from $\mathcal{E}^{trans}_i$, and  $\mathcal{T}^{total}$ from $\mathcal{E}^{total}$.  Thus we can state, $\forall_{i=1,...,|\mathcal{T}^{total}|} : \mathcal{T}^{trans}_i \backslash \mathcal{T}_i = \emptyset$ for the worst case defined by Alg. \ref{alg:adversarialStrategy}.
\end{proof}

%% In terms of accuracy, as defined in Equation \eqref{eq:accuracy}, that means that in the worst-case, to be assured perfect accuracy, one would need to be given all triplets, or $|\mathcal{T}^{total}|$ triplets.  To be assured accuracy of at least $\tau$, one would need to gather $\lceil\tau*|\mathcal{T}^{total}|\rceil$ triplets in the worst case.
%% Things I need to discss:
%% 1) Does this give ALL edges? (The loops go over all vertices and add all edges involving them)
%% 2) Does this every give inferred edges?  When?  Why are they not ones we havent been given.
%% 3) In the first loop, how will I know that when I add edges to the current sink vertex that I will only infer edges I have already given?
%%    a) Why do the minimum maximum thing?  (Working your way through subgraphs)
%% 4) In the second loop, how does this not cause infered edges that we already don't have (obvious, really).

\subsection{Theorem 3}
\label{sec:appTheorem3}
We begin by stating that the RCKL problem is equivalent to learning an embedding of objects in a space that satisfies distance constraints imposed by the triplets $\mathcal{T}$.  Limiting the comparisons to triplets makes this embedding problem a special case of non-metric multidimensional scaling (NMDS).  Also, the rank of the learned kernel $\mathbf{K}$ is equivalent to the rank of the learned embedding (i.e. if $\mathbf{K} = \mathbf{A}\mathbf{A}^T$, then $\mathrm{rank}(\mathbf{K}) = \mathrm{rank}(\mathbf{A})$.  As a result, enforcing $\mathrm{rank}(\mathbf{K}) \leq r$ is equivalent to enforcing that the objects be embedded in $\mathbb{R}^d$ where $d \leq r$.  With this in mind we state the following proposition from Appendix A of \cite{mcfee2011learning}:

\begin{proposition}
Any set of objects $\mathcal{X}$ with a partial order of distances $\mathcal{C}$ can be embedded in $\mathbb{R}^{n-1}$
\label{prop:embedding}
\end{proposition}

When $\mathcal{C}$ is a total ordering over all pairs, the problem of embedding $\mathcal{X}$ in a space that respects the distances in $\mathcal{C}$ reduces to NMDS \cite{kruskal1964nonmetric}.  This implies that given any objects $\mathcal{X}$ and any non-conflicting set of triplets $\mathcal{T}$, the objects in $\mathcal{X}$ can be embedded in $\mathbb{R}^{n-1}$ and satisfy all triplets in $\mathcal{T}$.  Equivalently, given any set of objects $\mathcal{X}$ and any non-conflicting set of triplets $\mathcal{T}$, all triplets in $\mathcal{T}$ can be satisfied by a rank $n-1$ kernel. With this in mind we can prove Thm. \ref{thm:noRank} from Section \ref{sec:theory}:

\begin{proof}
Assume an RCKL method enforces $\mathrm{rank}(\mathbf{K}) = r$.  By Prop. \ref{prop:embedding}, $r$ must be less than $n-1$ to infer any triplets.  Without loss of generality let $t = (a,b,c) \in \mathcal{T}^{rank-r}$.  Because, $t \notin \mathcal{T}^{trans}$ by definition of $\mathcal{T}^{rank-r}$, $\mathcal{T} \cup (a,c,b)$ does not cause a conflict.  An adversary can construct $\mathcal{T}^{total}$ with rank $r^{total} > r$,  such that $(a,c,b) \in \mathcal{T}^{total}$, because any non-conflicting set of triplets $\mathcal{T}^{total}$ can be satisfied by some choice of $r^{total} \leq n-1$ from Prop. \ref{prop:embedding}.  By virtue of the fact that $\mathcal{T}^{total}$ contains no conflicting triplets, we can then deduce that $(a,b,c) \notin \mathcal{T}^{total}$.  Trivially, this can be said of any rank less than $r$, as well.  Thus, any triplet $t \in \mathcal{T}^{rank-r}$ is not an element in $\mathcal{T}^{total}$ through adversarial choice of $\mathcal{T}^{total}$, proving Thm. 3.
\end{proof}

\section{Proofs of Propositions}
The strategy employed throughout this section to prove the stated functions are convex is to build each using convex combinations of convex functions.  In order to use this strategy, we need to establish the following Lems. (Above each Lemma is a reference to a source for each lemma, respectively). \\

\noindent Section 3.2.1 of \cite{boyd2004convex}:
\vspace{-5 pt}
\begin{lemma}
  If $f$ and $g$ are both convex functions, then so is their sum $f$ + $g$.
  \label{lemma:sumConvex} 
\end{lemma}

\noindent Section 2.3.2 of \cite{boyd2004convex}:
\vspace{-5 pt}
\begin{lemma}
  Affine functions of the form $f(\mathbf{x}) = \mathbf{Ax} + \mathbf{b}$, where $\mathbf{A} \in \mathbb{R}^{m\mathrm{x}n}$, $\mathbf{x} \in \mathbb{R}^n$, and $\mathbf{b} \in \mathbb{R}^m$ are convex in $\mathbf{x}$ 
\label{lemma:affineMap}
\end{lemma}

\noindent Section 3.2.3 of \cite{boyd2004convex}:
\vspace{-5 pt}
\begin{lemma}
  If $f$ and $g$ are convex functions, then their point-wise maximum, $\max\left(f(x), g(x)\right)$, is also convex.
\label{lemma:max}
\end{lemma}

\noindent In addition, we will use the concept of \emph{logarithmic convexity}:

\begin{definition}
  A function $f$ is \emph{logarithmically convex (log-convex)} if $f(x) > 0$  for all $x \in \mathbf{dom} f$ and $\log f$ is convex.
\label{def:logconvex} 
\end{definition}

\noindent Which we then use to state the following Lemma. \\

\noindent Section 3.5.2 of \cite{boyd2004convex}:
\vspace{-5 pt}
\begin{lemma}
  If $f$ and $g$ are both log-convex functions, then so is their sum $f$ + $g$.
  \label{lemma:sumlog}
\end{lemma}

\noindent Finally, for the sake of notational brevity, let us define the following:

\begin{align}
  d^{ab}_{\mathbf{K}} &= d_{\mathbf{K}}(\mathit{x}_a, \mathit{x}_b) \nonumber \\
  D^{abc}_{\mathbf{K}} &= d_{\mathbf{K}}(\mathit{x}_a, \mathit{x}_b) - d_{\mathbf{K}}(\mathit{x}_a, \mathit{x}_c)\nonumber
\end{align}

\noindent These short-hand versions of our established notation will be used throughout this section.

\subsection{Proposition 1}
\label{sec:appProposition1}
Proof of Prop. 1:
\begin{proof}
~\begin{enumerate}
\item In order for \eqref{eq:RCKL-AK} to be a convex optimization problem, it's objective and constraints must be convex in the optimization variables.
\item By Lem. \ref{lemma:sumConvex}, if $E(\mathbf{K}'', \mathcal{T})$, $\lambda_1\mathrm{trace}(\mathbf{K}_0)$, and $\lambda_2\|\mathbfup{\mu}\|_1$ are convex, then the objective in \eqref{eq:RCKL-AK} is convex.
\item It is an assumption of the proposition that $E(\mathbf{K}'', \mathcal{T})$ is convex.
\item $\lambda_1\mathrm{trace}(\mathbf{K}_0)$ is defined as a constant times the sum of the diagonal elements of the matrix $\mathbf{K}_0$, which is a sum of convex functions (Lem. \ref{lemma:sumConvex}).
\item $\lambda_2\|\mathbfup{\mu}\|_1$ is defined as a constant times the the sum of the absolute values of the elements of $\mathbfup{\mu}$, which is a sum of convex functions (Lem. \ref{lemma:sumConvex})
\item By lines 2, 3, 4, and 5, the objective in \eqref{eq:RCKL-AK} is convex.
\item The positivity constraint on $\mathbfup{\mu}$ is trivially convex.
\item The positive semidefinite constraint is known to be convex \cite{vandenberghe1996semidefinite}.
\item By lines 7 and 8, both constraints of \eqref{eq:RCKL-AK} are convex.
\item By lines 1, 6, and 9, \eqref{eq:RCKL-AK} is a convex optimization problem.
\end{enumerate}
\end{proof}

\subsection{Proposition 2}
\label{sec:prop2}
Proof of Prop. 2:
\begin{proof}
~\begin{enumerate}
\item Moving the negation into the sum, \eqref{eq:STE-AK} becomes the sum of terms of the following form:

%\begin{equation}
  \begin{align}
    & \hspace{5 pt}-\log\left(p_{abc}^{\mathbf{K}''}\right)  \label{eq:negSum}\\
    = &  \hspace{5 pt}-\log\left(\frac{\mathrm{exp}\left(-d^{ab}_{\mathbf{K}''}\right)}{\mathrm{exp}\left(-d^{ac}_{\mathbf{K}''}\right) + \mathrm{exp}\left(-d^{ab}_{\mathbf{K}''}\right)}\right) \nonumber \\ 
    = & \hspace{5 pt}-\log\left(\frac{1}{1 + \mathrm{exp}\left(D^{abc}_{\mathbf{K}''}\right)}\right) \nonumber \\ 
    = & \hspace{5 pt}-\log(1) \label{eq:1st}\\
    + & \hspace{5 pt}\log\left(\mathrm{exp}\left(D^{abc}_{\mathbf{K}''}\right) + 1\right) \label{eq:2nd}
  \end{align}
%\end{equation}

\item By Lem. \ref{lemma:sumConvex}, if \eqref{eq:negSum} is convex for all triplets $(a,b,c)$, then \eqref{eq:STE-AK} is convex.
  \item By Lem. \ref{lemma:sumConvex}, \eqref{eq:negSum} is convex if both \eqref{eq:1st} and \eqref{eq:2nd} are convex.
  \item \eqref{eq:1st} is a constant and trivially convex.
  \item By Definition \ref{def:logconvex}, if $\mathrm{exp}\left(D^{abc}_{\mathbf{K}''}\right) + 1$ is log-convex, then \eqref{eq:2nd} is convex.
  \item By Lem. \ref{lemma:sumlog} if $\mathrm{exp}\left(D^{abc}_{\mathbf{K}''}\right)$ and 1 are both log-convex, then $\mathrm{exp}\left(D^{abc}_{\mathbf{K}''}\right) + 1$ is log-convex.
  \item 1 is a constant and trivially log-convex.
  \item The codomain of the exponential function is $\mathbb{R}^+$, so $\mathrm{exp}\left(D^{abc}_{\mathbf{K}''}\right) > 0$ for all $\mathbf{K}''$, which satisfies the first condition for log-convexity.
  \item To show $\log\left(\mathrm{exp}\left(D^{abc}_{\mathbf{K}''}\right)\right)$ is convex, thus satisfying the second condition of log-convexity, we start by stating the following equivalence by using the definition of $\mathbf{K}''$ \eqref{eq:akcomb}:
%\begin{equation}
%  \begin{array}{rl}
%
\begin{flalign}
  \log\left(\mathrm{exp}\left(D^{abc}_{\mathbf{K}''}\right)\right) & =  D^{abc}_{\mathbf{K}''} \nonumber \\
  & = D^{abc}_{\mathbf{K}_0} + \sum_{i=1}^A \mu_i D^{abc}_{\mathbf{K}_i} \label{eq:kDists}
\end{flalign}
%  \end{array}
%\end{equation}  
  \item By Lem. \ref{lemma:sumConvex} if $D^{abc}_{\mathbf{K}_0}$ and $\sum_{i=1}^A \mu_i D^{abc}_{\mathbf{K}_i}$ are convex, then \eqref{eq:kDists} is convex.
  \item Let $\mathbf{k}_{abc} = \left(D^{abc}_{\mathbf{K}_1},...,D^{abc}_{\mathbf{K}_A}\right)$.  Then, $\sum_{i=1}^A \mu_i D^{abc}_{\mathbf{K}_i} = \mathbf{k}_{abc}^T\mathbfup{\mu}$.
   \item  $\mathbf{k}_{abc}^T\mathbfup{\mu}$ is an affine function of $\mathbfup{\mu}$ that has the form $f(\mathbfup{\mu}) = \mathbf{A}\mathbfup{\mu} + \mathbf{b}$ where $\mathbf{A} = \mathbf{k}_{abc}^T$ and $\mathbf{b}$ is 0. 
   \item By Lem. \ref{lemma:affineMap} and the previous step, $\mathbf{k}_{abc}^T\mathbfup{\mu}$ is convex in $\mathbfup{\mu}$.
   \item Using the definition of kernel distance from \eqref{eq:KConstraints} (Note the slight change in notation: $\mathbf{K}_0^{ab}$ refers to the $a$th column and $b$th row of $\mathbf{K}_0$):
 \begin{align}
    D^{abc}_{\mathbf{K}_0} &= \mathbf{K}_0^{bb} + 2\mathbf{K}_0^{ac} - \mathbf{K}_0^{cc} - 2\mathbf{K}_0^{ab}
    \label{eq:Kdist}
 \end{align}
\item By Lem. \ref{lemma:sumConvex}, if the individual terms of \eqref{eq:Kdist} are convex then \eqref{eq:Kdist} is convex.

\item The individual terms of \eqref{eq:Kdist} are simply elements of $\mathbf{K}_0$ multiplied by scalers, which are convex in $\mathbf{K}_0$.
\item By lines 3-16,  \eqref{eq:negSum} is convex.
\item By lines 1, 2, and 17, \eqref{eq:STE-AK} is convex.
\end{enumerate}
\end{proof}

\subsection{Proposition 3}
\label{sec:appProposition3}
Proof of Prop. 3:
\begin{proof}
~\begin{enumerate}
\item By Lem. \ref{lemma:sumConvex} if $\max(0, D^{abc}_{\mathbf{K}''} + 1)$ is convex for any triplet $(a,b,c)$, then \eqref{eq:GNMDS-AK-Error} is convex.
\item By Lem. \ref{lemma:max}, if 0 and $D^{abc}_{\mathbf{K}''} + 1$ are convex, then $\max(0, D^{abc}_{\mathbf{K}''} + 1)$ is convex.
\item 0 is trivially convex.
\item By Lem. \ref{lemma:sumConvex}, if 1 and $D^{abc}_{\mathbf{K}''}$ are convex, then $D^{abc}_{\mathbf{K}''} + 1$ is convex.
\item 1 is trivially convex.
\item Steps 9-16 of Section \ref{sec:prop2} showed $D^{abc}_{\mathbf{K}''}$ is convex in the optimization variables.
\item By lines 2-6, $\max(0, D^{abc}_{\mathbf{K}''} + 1)$ is convex for any triplet $(a,b,c)$.
\item By line 1 and 6, \eqref{eq:GNMDS-AK-Error} is convex
\end{enumerate}
\end{proof}

\section{Auxiliary Kernel Algorithms}
\label{sec:appAKA}
In this section we state and discuss the algorithms used to solve STE-AK and GNMDS-AK.  They share many steps, so we begin stating STE-AK and discussing it in detail, and then state GNMDS-AK and highlight how it differs from STE-AK.
\subsection{STE-AK}

After initialization, Alg. \ref{alg:STE-AK} repeats the following steps until convergence:
\begin{enumerate}
    \item \textbf{Line 6:} Take a gradient step for $\mathbf{K}_0$ (trace regularization included)
    \item \textbf{Line 7:} Take a gradient step for $\mathbfup{\mu}$ ($\ell_1$-norm regularization included)
    \item \textbf{Line 8:} Project $\mathbf{K}_0$ onto the positive semidefinite cone
    \item \textbf{Line 9:} Project the elements of $\mathbfup{\mu}$ to be non-negative
    \item \textbf{Line 10:} Update $\mathbf{K}''$
\end{enumerate}

\begin{algorithm}
\center \caption{STE-AK Projected Gradient Descent}
\begin{algorithmic}[1]

\Require
\Statex $\mathcal{X} = \{\mathit{x}_1, ..., \mathit{x}_n\}$,
\Statex $\mathcal{T} = \{(a,b,c) | \mathit{x}_a \mathrm{\ is\ more\ similar\ to\ }\mathit{x}_b \mathrm{\ than\ }\mathit{x}_c\}$,
\Statex $\mathbf{K}_1,...,\mathbf{K}_A \in \mathbb{R}^{n \times n}$, $\lambda_1 \in \mathbb{R}^{+}$, $\lambda_2 \in \mathbb{R}^{+}$, $\eta \in \mathbb{R}^{+}$
\Output 
\Statex $\mathbf{K}'' \in \mathbb{R}^{n \times n}$
\Statex
\State $t \gets 0$
\State $\mathbf{K}^0_0 \gets \mathbf{I}^{n \times n}$
\State $\mu_1^0,...,\mu_A^0 \gets \frac{1}{A}$
\State $\mathbf{K}'' \gets \mathbf{K}^0_0 + \sum_{a=1}^A\mu_a^0\mathbf{K}_a$
\Repeat
    \State $\mathbf{K}^{t+1}_0 \gets \mathbf{K}^t - \eta * (\nabla_{\mathbf{K}^t}E_{\mathrm{STE}}(\mathbf{K}'', \mathcal{T})  + \lambda_1 * \mathbf{I}^{n \times n})$
    \State $\mathbfup{\mu}^{t+1} \gets \mathbfup{\mu}^t - \eta * (\nabla_{\mathbfup{\mu}^t}E_{\mathrm{STE}}(\mathbf{K}'', \mathcal{T}) + \mathbfup{\lambda}_2)$
    \State $\mathbf{K}^{t+1}_0 \gets \Pi_{PSD}(\mathbf{K}^{t+1}_0)$
    \State $\mathbfup{\mu}^{t+1} \gets \Pi_+(\mathbfup{\mu}^{t+1})$    
    \State $\mathbf{K}'' \gets \mathbf{K}^{t+1}_0 + \sum_{a=1}^A\mu_a^{t+1}\mathbf{K}_a$
    \State $t \gets t + 1$ 
\Until{convergence}
\end{algorithmic}
\label{alg:STE-AK}
\end{algorithm}

Projection onto the positive semi-definite cone is done by performing eigendecomposition of the matrix $\mathbf{K}_0$, assigning all negative eigenvalues to zero, and then reassembling $\mathbf{K}_0$ from the original eigenvectors and the new eigenvalues \cite{schwertman1979smoothing}. Projection of the elements of $\mathbfup{\mu}$ to be non-negative is simply done by assigning all negative elements to be zero. The $\ell_1$-norm regularization in Alg. \ref{alg:STE-AK} is performed by adding $\mathbfup{\lambda}_2 = \lambda_2 * \mathbfup{1}^{A}$ to the gradient (Line 7).  Since $\mathbfup{\mu}$ is constrained to the non-negative orthant, the subgradient of the $\ell_1$-norm function needs only to be over the non-negative orthant, thus $\mathbfup{\lambda}_2$ is an acceptable subgradient.  Moreover, since we then project the elements of $\mathbfup{\mu}$ to be non-negative, we get the desired effect of the $\ell_1$-norm regularization: the reduction of some elements to be exactly zero.  

\subsection{GNMDS-AK}
\begin{algorithm}
\center \caption{GNMDS-AK Projected Gradient Descent}

\begin{algorithmic}[1]
\Require
\Statex $\mathcal{X} = \{\mathit{x}_1, ..., \mathit{x}_n\}$,
\Statex $\mathcal{T} = \{(a,b,c) | \mathit{x}_a \mathrm{\ is\ more\ similar\ to\ }\mathit{x}_b \mathrm{\ than\ }\mathit{x}_c\}$,
\Statex $\mathbf{K}_1,...,\mathbf{K}_A \in \mathbb{R}^{n \times n}$, $\lambda_1 \in \mathbb{R}^{+}$, $\lambda_2 \in \mathbb{R}^{+}$, $\eta \in \mathbb{R}^{+}$
\Output 
\Statex $\mathbf{K}'' \in \mathbb{R}^{n \times n}$
\Statex
\State $t \gets 0$
\State $\mathbf{K}^0_0 \gets \mathbf{I}^{n \times n}$
\State $\mu_1^0,...,\mu_A^0 \gets \frac{1}{A}$
\State $\mathbf{K}'' \gets \mathbf{K}^0_0 + \sum_{a=1}^A\mu_a^0\mathbf{K}_a$
\Repeat
    \State $\mathcal{T}' \gets getActiveTriplets\left(\mathcal{T},\mathbf{K}''\right)$ 
    \State $\mathbf{K}^{t+1}_0 \gets \mathbf{K}^t - \eta * (\nabla_{\mathbf{K}^t}E_{\mathrm{GNMDS}}^{'}(\mathbf{K}'', \mathcal{T}')  + \lambda_1 * \mathbf{I}^{n \times n})$
    \State $\mathbfup{\mu}^{t+1} \gets \mathbfup{\mu}^t - \eta * (\nabla_{\mathbfup{\mu}^t}E_{\mathrm{GNMDS}}^{'}(\mathbf{K}'', \mathcal{T}') + \mathbfup{\lambda}_2)$
    \State $\mathbf{K}^{t+1}_0 \gets \Pi_{PSD}(\mathbf{K}^{t+1}_0)$
    \State $\mathbfup{\mu}^{t+1} \gets \Pi_+(\mathbfup{\mu}^{t+1})$    
    \State $\mathbf{K}'' \gets \mathbf{K}^{t+1}_0 + \sum_{a=1}^A\mu_a^{t+1}\mathbf{K}_a$
    \State $t \gets t + 1$ 
\Until{convergence}
\end{algorithmic}
\label{alg:GNMDS-AK}
\end{algorithm}

The differences between Algs. \ref{alg:STE-AK} and \ref{alg:GNMDS-AK} are in lines 6, 7, and 8. 

\begin{enumerate}
    \item \textbf{Line 6:} The function $getActiveTriplets$ checks each triplet in $\mathcal{T}$ to see if it violates the margin constraint $d_{\mathbf{K}''}(\mathit{x}_a, \mathit{x}_c) - d_{\mathbf{K}''}(\mathit{x}_a, \mathit{x}_b) < 1$.  The set $\mathcal{T}'$ is assigned to be the set containing only triplets that violate this constraint.  What this effectively does is form a set of only the triplets that would contribute to the error $E_{\mathrm{GNMDS}}$, because the point-wise maximum would be non-zero.
    \item \textbf{Lines 7 and 8:} These two lines are similar to lines 6 and 7 in Alg. \ref{alg:GNMDS-AK}, respectively. The first difference is in the error function  $E_{\mathrm{GNMDS}}^{'}$:

\begin{equation*}
E_{\mathrm{GNMDS}}^{'}\left(\mathbf{K}, \mathcal{T}\right) = \hspace{-10 pt}\displaystyle \sum_{(a,b,c) \in \mathcal{T}} \hspace{-10 pt} d_{\mathbf{K}}(\mathit{x}_a, \mathit{x}_b) - d_{\mathbf{K}}(\mathit{x}_a, \mathit{x}_c) + 1 
\end{equation*}

This is simply the sum over the triplets of the second argument in the point-wise maximum of $E_{\mathrm{GNMDS}}$.  Because the second argument of $E_{\mathrm{GNMDS}}^{'}$ in these lines is the set $\mathcal{T}'$, $E_{\mathrm{GNMDS}}^{'}\left(\mathbf{K}'', \mathcal{T}'\right)$ is equivalent to $E_{\mathrm{GNMDS}}\left(\mathbf{K}'', \mathcal{T}\right)$  for all possible triplets, but is differentiable everywhere.
\end{enumerate}

\section{Derivation of Equation \eqref{eq:metricMap}} 
\label{sec:appDerivation}
We begin by stating the definition of the distance metric learned in Equation (5) of \cite{wang2011metric}:

\begin{equation}
  \begin{array}{ll}
    d^2_{\mathbf{A},\mathbfup{\mu}}(x_i, x_j) \hspace{-5pt}& =  \displaystyle \sum_{l=1}^{A}\mu_l\left(\mathbf{K}_l^i - \mathbf{K}_l^j\right)^T \mathbf{A} \sum_{l=1}^{A}\mu_l\left(\mathbf{K}_l^i - \mathbf{K}_l^j\right)  \\ 
    \multicolumn{2}{l}{= \left(\mathbfup{\Phi}_{\mathbfup{\mu}}\left(x_i\right) - \mathbfup{\Phi}_{\mathbfup{\mu}}\left(x_j\right)\right)^T\mathbfup{\Phi}_{\mathbfup{\mu}}\left(\mathcal{X}\right)^T\mathbf{A}\mathbfup{\Phi}_{\mathbfup{\mu}}\left(\mathcal{X}\right)\left(\mathbfup{\Phi}_{\mathbfup{\mu}}\left(x_i\right) - \mathbfup{\Phi}_{\mathbfup{\mu}}\left(x_j\right)\right)} 
  \end{array}
\label{eq:wangMetric}
\end{equation}

Here, $\mathbf{K}_l^i$ is the $i$th row of the $l$th auxiliary kernel, and $\mathbf{A} \in \mathbb{R}^{n\mathrm{x}n}$.  Also:

\begin{equation}
  \mathbfup{\Phi}_{\mathbfup{\mu}}\left(x\right) = \left[\sqrt{\mu_1}\mathbfup{\Phi}_{1}\left(x\right),...,\sqrt{\mu_A}\mathbfup{\Phi}_{A}\left(x\right)\right]
\label{eq:wangMap}
\end{equation}

Finally, $\mathbfup{\Phi}_{\mathbfup{\mu}}\left(\mathcal{X}\right)$ is the matrix in which the rows are the mappings of the elements of $\mathcal{X}$ to $\mathcal{H}_{\mathbfup{\mu}}$ by $\mathbfup{\Phi}_{\mathbfup{\mu}}$.  Note the slight change in notation from \cite{wang2011metric}. The domain of their mapping $\mathbfup{\Phi}_{\mathbfup{\mu}}$ is over vector representations of the objects ($\mathbf{x}$).  In this work we assume that the mapping can exist over elements in a set ($x$) that do not necessarily need to be vectors.  By factoring the matrix $\mathbf{A} = \mathbf{B}^T\mathbf{B}$ and defining $\mathbf{\Omega} = \mathbf{B}\mathbfup{\Phi}_{\mathbfup{\mu}}\left(\mathcal{X}\right)$ we can distribute $\mathbf{\Omega}$ through \eqref{eq:wangMetric}:

\begin{equation}
  \begin{array}{ll}
    d^2_{\mathbf{A},\mathbfup{\mu}}(x_i, x_j) \hspace{-5pt} &  \\ 
    \multicolumn{2}{l}{= \left(\mathbfup{\Phi}_{\mathbfup{\mu}}\left(x_i\right) - \mathbfup{\Phi}_{\mathbfup{\mu}}\left(x_j\right)\right)^T\mathbfup{\Omega}^T\mathbfup{\Omega}\left(\mathbfup{\Phi}_{\mathbfup{\mu}}\left(x_i\right) - \mathbfup{\Phi}_{\mathbfup{\mu}}\left(x_j\right)\right)}  \\
    \multicolumn{2}{l}{= \left(\mathbfup{\Omega}\mathbfup{\Phi}_{\mathbfup{\mu}}\left(x_i\right) - \mathbfup{\Omega}\mathbfup{\Phi}_{\mathbfup{\mu}}\left(x_j\right)\right)^T\left(\mathbfup{\Omega}\mathbfup{\Phi}_{\mathbfup{\mu}}\left(x_i\right) - \mathbfup{\Omega}\mathbfup{\Phi}_{\mathbfup{\mu}}\left(x_j\right)\right)} \\
    \multicolumn{2}{l}{= d^2\left(\mathbfup{\Omega}\mathbfup{\Phi}_{\mathbfup{\mu}}\left(x_i\right), \mathbfup{\Omega}\mathbfup{\Phi}_{\mathbfup{\mu}}\left(x_j\right)\right)}
  \end{array}
\label{eq:wangmMetric}
\end{equation}

Where $d^2\left(\mathbf{x}_i, \mathbf{x}_j\right)$ is the squared Euclidean distance between $\mathbf{x}_i$ and $\mathbf{x}_j$.  In this form and by \eqref{eq:wangMap}, we can see that learning the proposed Mahalanobis distance metric in Equation (5) of \cite{wang2011metric} is equivalent to learning a squared Euclidean distance of points transformed by the following linear transformation:

\begin{align}
\mathbfup{\Phi}_{\mathbfup{\mu}, \mathbfup{\Omega}}(x) & = \mathbfup{\Omega}\mathbfup{\Phi}_{\mathbfup{\mu}}\left(x\right) \nonumber\\
& = \mathbfup{\Omega}\left[\sqrt{\mu_1}\mathbfup{\Phi}_1(x),...,\sqrt{\mu_A}\mathbfup{\Phi}_A(x)\right] \nonumber
\label{eq:metricMap}
\end{align} 

Which is the mapping defined in \eqref{eq:metricMap} from Section \ref{sec:related}.

\bibliography{RCKL-AK}

\begin{thebibliography}{10}
\providecommand{\url}[1]{\texttt{#1}}
\providecommand{\urlprefix}{URL }

\bibitem{agarwal2007generalized}
Agarwal, S., Wills, J., Cayton, L., Lanckriet, G., Kriegman, D., Belongie, S.:
  Generalized non-metric multidimensional scaling. In: AISTATS (2007)

\bibitem{bengio2004out}
Bengio, Y., Paiement, J., Vincent, P., Delalleau, O., Roux, N.L., Ouimet, M.:
  Out-of-sample extensions for lle, isomap, mds, eigenmaps, and spectral
  clustering. NIPS  16 (2004)

\bibitem{boyd2004convex}
Boyd, S., Vandenberghe, L.: Convex optimization. Cambridge university press
  (2004)

\bibitem{davis2007information}
Davis, J., Kulis, B., Jain, P., Sra, S., Dhillon, I.: Information-theoretic
  metric learning. In: ICML (2007)

\bibitem{ellis2002quest}
Ellis, D., Whitman, B., Berenzweig, A., Lawrence, S.: The quest for ground
  truth in musical artist similarity. In: ISMIR (2002)

\bibitem{filippone2008survey}
Filippone, M., Camastra, F., Masulli, F., Rovetta, S.: A survey of kernel and
  spectral methods for clustering. Pattern Recognition  41(1),  176--190 (2008)

\bibitem{gonen2011multiple}
G{\"o}nen, M., Alpayd{\i}n, E.: Multiple kernel learning algorithms. JMLR  12,
  2211--2268 (2011)

\bibitem{huang2011generalized}
Huang, K., Ying, Y., Campbell, C.: Generalized sparse metric learning with
  relative comparisons. KAIS  28(1),  25--45 (2011)

\bibitem{Jain12}
Jain, A., Vishwanathan, S., Varma, M.: Spg-gmkl: Generalized multiple kernel
  learning with a million kernels. In: SIGKDD (August 2012)

\bibitem{kendall1990rank}
Kendall, M., Gibbons, J.: Rank Correlation Methods. Oxford University Press,
  fifth edn. (1990)

\bibitem{kruskal1964nonmetric}
Kruskal, J.: Nonmetric multidimensional scaling: a numerical method.
  Psychometrika  29(2),  115--129 (1964)

\bibitem{kulis2006learning}
Kulis, B., Sustik, M., Dhillon, I.: Learning low-rank kernel matrices. In: ICML
  (2006)

\bibitem{lanckriet2004learning}
Lanckriet, G., Cristianini, N., Bartlett, P., Ghaoui, L., Jordan, M.: Learning
  the kernel matrix with semidefinite programming. JMLR  5,  27--72 (2004)

\bibitem{van2012stochastic}
van~der Maaten, L., Weinberger, K.: Stochastic triplet embedding. In: MLSP
  (2012)

\bibitem{mcfee2009heterogeneous}
McFee, B., Lanckriet, G.: Heterogeneous embedding for subjective artist
  similarity. In: ISMIR (2009)

\bibitem{mcfee2011learning}
McFee, B., Lanckriet, G.: Learning multi-modal similarity. JMLR  12,  491--523
  (2011)

\bibitem{rakotomamonjy2008simplemkl}
Rakotomamonjy, A., Bach, F., Canu, S., Grandvalet, Y.: Simplemkl. JMLR  9,
  2491--2521 (2008)

\bibitem{scholopf2002learning}
Sch{\"o}lkopf, B., Smola, A.: Learning with kernels: support vector machines,
  regularization, optimization and beyond. the MIT Press (2002)

\bibitem{schultz2004learning}
Schultz, M., Joachims, T.: Learning a distance metric from relative
  comparisons. NIPS  (2004)

\bibitem{schwertman1979smoothing}
Schwertman, N., Allen, D.: Smoothing an indefinite variance-covariance matrix.
  JSCS  9(3),  183--194 (1979)

\bibitem{shepard1962analysis}
Shepard, R.: The analysis of proximities: Multidimensional scaling with an
  unknown distance function. i. Psychometrika  27(2),  125--140 (1962)

\bibitem{tamuz2011adaptively}
Tamuz, O., Liu, C., Belongie, S., Shamir, O., Kalai, A.: Adaptively learning
  the crowd kernel. ICML  (2011)

\bibitem{tibshirani1996regression}
Tibshirani, R.: Regression shrinkage and selection via the lasso. JRSS. Series
  B pp. 267--288 (1996)

\bibitem{valizadegan2006generalized}
Valizadegan, H., Jin, R.: Generalized maximum margin clustering and
  unsupervised kernel learning. In: NIPS (2006)

\bibitem{vandenberghe1996semidefinite}
Vandenberghe, L., Boyd, S.: Semidefinite programming. SIAM review  38(1),
  49--95 (1996)

\bibitem{varma2009more}
Varma, M., Babu, B.: More generality in efficient multiple kernel learning. In:
  ICML (2009)

\bibitem{wang2008label}
Wang, F., Zhang, C.: Label propagation through linear neighborhoods. TKDE
  20(1),  55--67 (2008)

\bibitem{wang2011metric}
Wang, J., Do, H., Woznica, A., Kalousis, A.: Metric learning with multiple
  kernels. In: NIPS (2011)

\bibitem{zhou2004learning}
Zhou, D., Bousquet, O., Lal, T., Weston, J., Sch{\"o}lkopf, B.: Learning with
  local and global consistency (2004)

\bibitem{zhu2002learning}
Zhu, X., Ghahramani, Z.: Learning from labeled and unlabeled data with label
  propagation. Tech. rep., CMU-CALD-02-107, Carnegie Mellon University (2002)

\end{thebibliography}
\bibliographystyle{splncs03}

\end{document}